%% file: main.tex
\documentclass[accepted]{uai2021} 

\usepackage[american]{babel}

\usepackage{natbib} 
\usepackage{mathtools} 
\usepackage{booktabs} 
\usepackage{tikz} 

\usepackage{natbib}
\usepackage{hyperref}       
\usepackage{url}            
\usepackage{amsfonts}       
\usepackage{nicefrac}       
\usepackage{microtype}      
\usepackage{bbold}
\usepackage{capt-of}
\usepackage{amsmath, amsthm, amssymb, graphicx, cite, epsfig,epstopdf,color,soul}
\usepackage{tabularx}
\usepackage[ruled,vlined]{algorithm2e}
\allowdisplaybreaks
\usepackage{color}
\usepackage{enumerate}
\usepackage{multicol}
\usepackage{empheq}
\usepackage{caption} 
\usepackage{float}

\input{header}

\input{defs-mlmath}

\title{A Decentralized Policy Gradient Approach to Multi-Task Reinforcement Learning}

\author[1]{Sihan Zeng}
\author[1]{Malik Aqeel Anwar}
\author[2]{Thinh T. Doan}
\author[1]{Arijit Raychowdhury}
\author[1]{Justin Romberg}
\affil[1]{%
    School of Electrical and Computer Engineering\\
    Georgia Institute of Technology\\
    Atlanta, Georgia, USA
}
\affil[2]{%
    Bradley Department of Electrical and Computer Engineering\\
    Virginia Tech\\
    Blacksburg Virginia, USA
}

\begin{document}
\maketitle

\begin{abstract}
We develop a mathematical framework for solving multi-task reinforcement learning (MTRL) problems based on a type of policy gradient method. The goal in MTRL is to learn a common policy that operates effectively in different environments; these environments have similar (or overlapping) state spaces, but have different rewards and dynamics. We highlight two fundamental challenges in MTRL that are not present in its single task counterpart, and illustrate them with simple examples. We then develop a decentralized entropy-regularized policy gradient method for solving the MTRL problem, and study its finite-time convergence rate. We demonstrate the effectiveness of the proposed method using a series of numerical experiments.  These experiments range from small-scale "GridWorld" problems that readily demonstrate the trade-offs involved in multi-task learning to large-scale problems, where common policies are learned to navigate an airborne drone in multiple (simulated) environments.

\end{abstract}

\input{Introduction}

\input{MTRL_Intro}

\input{Decentralized_PG}

\input{Experiments}

\input{Conclusion}

\clearpage
\bibliography{references}

\appendix

\newpage
\appendix
\onecolumn
\vbox{%
\hsize\textwidth
\linewidth\hsize
\hrule height 4pt
\vskip 0.25in
\centering
{\Large\bf{A Decentralized Policy Gradient Approach to Multi-Task Reinforcement Learning \\ 
Supplementary Materials} \par}
  \vskip 0.25in
\hrule height 1pt
\vskip 0.09in
}
\input{Appendix_Examples}

\input{Appendix_Proof}

\input{Appendix_Experiments}

\end{document}

%% file: header.tex
\newcommand{\Eset}{\mathbb{E}}

\newcommand{\Rset}{\mathbb{R}}

\newcommand{\Acal}{{\cal A}}

\newcommand{\Ecal}{{\cal E}}

\newcommand{\Gcal}{{\cal G}}

\newcommand{\Mcal}{{\cal M}}
\newcommand{\Ncal}{{\cal N}}
\newcommand{\Ocal}{{\cal O}}
\newcommand{\Pcal}{{\cal P}}

\newcommand{\Rcal}{{\cal R}}
\newcommand{\Scal}{{\cal S}}

\newcommand{\Vcal}{{\cal V}}

\renewcommand{\arraystretch}{1.1}

\newtheorem{prop}{Proposition}

\newtheorem{thm}{Theorem}

\newtheorem{assump}{Assumption}
\newtheorem{remark}{Remark}

%% file: defs-mlmath.tex
\renewcommand{\j}{j}

\newcommand{\vct}[1]{\boldsymbol{#1}}
\newcommand{\mtx}[1]{\boldsymbol{#1}}

\newcommand{\set}[1]{\mathcal{#1}}

\newcommand{\vu}{\vct{u}}

\newcommand{\vx}{\vct{x}}

\newcommand{\vmu}{\vct{\mu}}

\newcommand{\vtheta}{\vct{\theta}}
\newcommand{\vrho}{\vct{\rho}}
\newcommand{\vxi}{\vct{\xi}}
\newcommand{\mI}{\mtx{I}}

\newcommand{\mL}{\mtx{L}}

\newcommand{\mV}{\mtx{V}}

\newcommand{\setA}{\set{A}}

\newcommand{\setS}{\set{S}}

\newtheorem{theorem}{Theorem}[section]

\newtheorem{lemma}[theorem]{Lemma}

\usepackage{tikz}

%% file: Introduction.tex
\section{INTRODUCTION}
In reinforcement learning (RL), an agent tries to learn an optimal policy through repeated interactions with its environment, modeled as a Markov decision process (MDP), with the goal of optimizing its long-term cumulative rewards.  
Combined with powerful function approximation such as neural networks, (deep) reinforcement learning has received great successes in solving challenging problems in different applications, including game playing \citep{mnih2015human,silver2016mastering,openai2019dota}, healthcare \citep{yu2019reinforcement,esteva2019guide}, robotics \citep{kober2013reinforcement,Haarnoja_2019}, and autonomous navigation \citep{kretzschmar2016socially,zhu2017target,anwar2019autonomous}. These results, however, are primarily achieved only on a single task, and every new task almost requires the agent to be re-trained from scratch.

Multi-task reinforcement learning (MTRL) addresses this problem by finding a single policy that is simultaneously effective for a number of tasks. We are interested in developing a new method for MTRL by using a group of learning agents. We consider a scenario where multiple agents, each learning in its own environment, work together to learn a common policy by sharing their policy parameters. This common policy may perform slightly worse than the optimal policy for each local task, but is general enough to solve all local tasks reasonably well. In other words, the agent learns how to perform well not only in its own environment but also in unseen environments explored by other agents.



Existing approaches to solving problems of this nature \citep{hessel2019multi,espeholt2018impala,Yu_gradientSurgery2020} typically use a specific ``master/worker'' model for agent interaction, where worker agents independently collect observations in their respective environments, which are then summarized (perhaps through a gradient computation) and reported to a central master.
We are interested in understanding MTRL under a more flexible, decentralized communication model where agents only share information with a small subset of other agents.
This framework is inspired by applications where centralized coordination is unwieldy or impossible; one example would be a network of mobile robots exploring different parts of an area of interest that can only communicate locally. This question has not yet been addressed in the existing literature, and our focus, therefore, is to solve this important problem by developing a decentralized policy gradient method. 
%
%

\textbf{Main Contributions}\vspace{0.1cm}\\
$\bullet$ We present a clean mathematical formulation for MTRL problems over a network of agents, where each task is assigned to a single agent.  Framing the problem in the language of distributed optimization allows us to develop a decentralized policy gradient algorithm that finds a single policy that is  effective for each of the tasks.\vspace{0.1cm}\\
$\bullet$  We present, in Section~\ref{SEC:CHALLENGES_MTRL}, two simple examples that illustrate the fundamental differences between learning a policy for one task and learning for multiple tasks. \vspace{0.1cm}\\
$\bullet$ We provide theoretical guarantees for the performance of our decentralized policy gradient algorithm.  We show that in the tabular setting, the algorithm converges to a stationary point of the global (non-concave) objective.  Under a further assumption on the structure of environments' dynamics, the algorithm converges to the globally optimal value.\vspace{0.1cm}\\
$\bullet$ We demonstrate the effectiveness of the proposed method using numerical experiments on challenging MTRL problems.  Our small-scale ``Grid World'' problems, which can be reliably solved using a complete tabular representation for the policy, demonstrate how the decentralized policy gradient algorithm balances the interests of the agents in different environments.  Our experiments for learning to navigate airborne drones in multiple (simulated) environments show that the algorithm can be scaled to problems that require significant amounts of data and use neural network representations for the policy function.

\subsection{Related Works}
In recent years, multi-task RL has become an emerging topic as a way to scale up RL solutions. This topic has received a surge of interests, and a number of solutions have been proposed for solving this problem, including policy distillation
\citep{rusu2015policy,traore2019discorl}, distributed RL algorithms over actors/learner networks \citep{espeholt2018impala,hessel2019multi,liu2016decoding, Yu_gradientSurgery2020}, and transfer learning \citep{gupta2017learning,d'eramo2020sharing}. Distributed parallel computing has also been applied to speed up RL algorithms for solving single task problems  \citep{mnih2016asynchronous, Arun_massiveRL2015,assran2019gossip}.

Similar to our work, \citet{espeholt2018impala,hessel2019multi} also aim to solve MTRL with policy gradient algorithms in a distributed manner. These works propose sharing the local trajectories/data collected by workers in each environment to a centralized server where learning takes place. When the data dimension is large, the amount of information required to be exchanged could be enormous. In contrast, exchanging the policy parameters could be a more compact and efficient form of communication in applications with a large state representation but a much smaller policy representation. Moreover, we observe that a wide range of practical problems do not allow for a centralized communication topology \citep{ovchinnikov2014decentralized}. Motivated by these observations, we consider a decentralized policy gradient method where the agents only exchange their policy parameters according to a decentralized communication graph. This makes our work fundamentally different from the existing literature. Indeed, our work can be considered as a decentralized and multi-task variant of the policy gradient method studied in \citet{AgarwalKLM2019}, where the authors consider a single-task {RL}.


Other works in meta-learning and transfer learning also essentially aim to achieve MTRL, where these two methods essentially attempt to reduce the resources required to learn a new task by utilizing related existing information; see for example  \citet{wang2016learning,nagabandi2018learning,anwar2019autonomous}. Our work is fundamentally different from these papers, where we address MTRL by leveraging the collaboration between a number of agents.

We also note some relevant works on decentralized algorithms in multi-agent reinforcement learning (MARL), where a group of agents operate in a common environment and aim to solve a single task \citep{zhang2019distributed, Chu_MARL2020,Qu_ScalableRL2019,DoanMR2019_DTD(0),Ding_TDHomotopy_2019,Li_F2A2_2020,Wai2018_NIPS,Kar2013_QDLearning,Lee_MARLReview2019,Zhang_MARLReview2019}. The setting in these work is different from ours since we consider multi-task RL, which is more challenging than solving a single task.

%% file: MTRL_Intro.tex
\section{MULTI-TASK REINFORCEMENT LEARNING}

We consider an MTRL problem with $N$ agents operating in $N$ different environments.  These environments, each characterized by a different Markov decision process ({\sf MDP}) as described below, might also be interpreted as encoding a different task that an agent attempts to accomplish.  Although each agent acts and makes observations in a single environment, their goal is to learn a policy
that is jointly optimal across all of the environments.  Information is shared between the agents through connections described by edges in an undirected graph.  We do not require the state spaces to be the same in each of the environments; in general, the learned joint policy is a mapping from the union of state spaces to the action space.\looseness=-1
\vspace{-5pt}

\subsection{MTRL FORMULATION}
\vspace{-5pt}
The {\sf MDP} at agent $i$ is given by the $5$-tuple $\Mcal_{i} = (\Scal_{i},\Acal,\Pcal_{i},\Rcal_{i},\gamma_{i})$ 
where $\Scal_{i}$ is the set of states, $\Acal$ is the set of possible actions, which has to be common across tasks, $\Pcal_{i}$ is the transition probabilities that specify the distribution on the next state given the current state and an action, $\Rcal_{i}:\Scal_{i}\times\Acal\rightarrow\Rset$ is the reward function,
and $\gamma_{i}\in(0,1)$ is the discount factor. We denote by $\Scal = \cup_{i}\Scal_{i}$, where $\Scal_{i}$ can share common states. We focus on randomized stationary policies ({\sf RSPs}), where agent $i$ maintains a policy $\pi_{i}$ that assigns to each $s\in\Scal_i$ a probability distribution $\pi_{i}(\cdot|s)$ over $\Acal$. \looseness=-1

Given a policy $\pi$, let $V^{\pi}_i$ be the value function associated with the $i$-th environment,
\begin{equation}
V_i^{\pi}(s_{i}) = \Eset\left[\sum_{k=0}^{\infty}\gamma_{i}^{k}\Rcal_{i}(s^{k}_{i},a_{i}^{k})\,|\,s_{i}^{0} = s_{i}\right],\,\, a_{i}^{k} \sim \pi(\cdot|s_{i}^{k}).
\label{sec:prob:value_function}
\vspace{-5pt}
\end{equation}
Similarly, we denote by $Q_i^{\pi}$ and $A_{i}^{\pi}$ the $Q$-function and advantage function in the $i$-th environment
\begin{align}
&Q_i^{\pi}(s_{i},a_{i}) = \Eset\left[\sum_{k=0}^{\infty}\gamma_{i}^{k}\Rcal(s_{i}^{k},a_{i}^{k})\,|\,s_{i}^{0} = s_{i}, a_{i}^{0} = a_{i}\right],\notag\\ &A_i^{\pi}(s_{i},a_{i}) = Q_i^{\pi}(s_{i},a_{i})- V_i^{\pi}(s_{i}).
\label{sec:prob:QA_function}
\end{align} 
Without loss of generality, we  assume that $\Rcal_i(s,a)\in[0,1]$, implying for any policy $\pi$ and $\forall s\in\Scal_i,a\in\Acal$
\begin{equation}
\hspace{-0.0pt}    0\leq V_i^{\pi}(s)\leq\frac{1}{1-\gamma_i},\quad -\frac{1}{1-\gamma_i}\leq A_i^{\pi}(s,a)\leq\frac{1}{1-\gamma_i}\cdot\label{eq:bounded_V_A}
\vspace{-5pt}
\end{equation}

Let $\rho_{i}$ be an initial state distribution over $\setS_i$, and with some abuse of notation we denote the long-term reward associated with this distribution as $V_{i}^{\pi}(\rho_{i}) = \Eset_{s_{i}\sim\rho_{i}}\left[V_{i}^{\pi}(s_{i})\right]$. 

To parameterize the policy, we consider the scenario where each agent maintains $\theta_i\in\mathbb{R}^{|\setS|\times|\setA|}$ and uses the popular softmax parameterization\footnote{Note that our method also works with other forms of stochastic policies.}, i.e.
\begin{equation}
    \pi_{\theta_i}(a\,|\,s) = \frac{\exp\left(\theta_{i\,;\,s,a}\right)}{\sum_{a'\in\Acal}\exp(\theta_{i\,;\,s,a'})}\cdot \label{sec:prob:softmax}    
\end{equation}
The goal of the agents is to cooperatively find a parameter $\theta^*$ that maximizes the total cumulative discounted rewards 
\vspace{-3pt}
\begin{equation}
\hspace{-0.03cm}\theta^* = \arg\max_{\theta} V^{\pi_{\theta}}(\boldsymbol{\rho}) \triangleq \sum_{i=1}^{N}V_{i}^{\pi_{\theta}}(\rho_{i}),
\,\, \vrho = [\rho_1;\hdots;\rho_N],  \label{sec:prob:obj}
\vspace{-5pt}
\end{equation}
which is a non-concave objective \citep{AgarwalKLM2019}.

Treating each of the environments as independent {\sf RL} problems would produce different policies $\pi_{i}^*$, each maximizing their respective $V_i^\pi$.  Our focus in this paper is to find a single $\pi^*$ that balances the performance across all environments.

\input{MTRL_Challenge}

%% file: MTRL_Challenge.tex
\subsection{MTRL Challenges}
\label{SEC:CHALLENGES_MTRL}
While solving single task RL is well understood at least in the tabular settings, MTRL is more challenging than it appears from \eqref{sec:prob:obj}. Here, we provide two fundamental challenges of MTRL, which make this problem much more difficult than its single-task counterpart.

\textbf{Deterministic vs stochastic policies}. In single task RL it is known that under mild assumptions there exists a deterministic policy $\pi^*$ that maximizes the objective \citep{Puterman_book_1994}. In addition, the value function of the optimal deterministic policy satisfies the Bellman optimality equation, motivating the development of the popular Q-learning method. In MTRL, where each task operates under different dynamics (transition probability matrices), there need not be an optimal deterministic policy, and hence there is no natural analog to the Bellman equation.  We illustrate this below with a simple ``GridWorld'' example.

\begin{figure}
  \includegraphics[width=\linewidth]{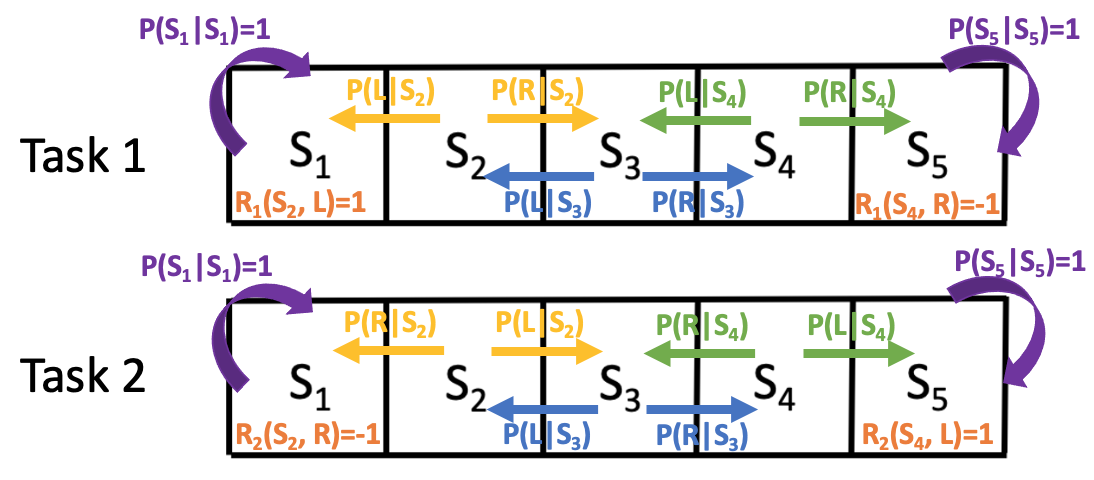}
  \vspace{-.7cm}
  \caption{2-Task GridWorld Problem}
  \label{fig:counterexample}
  \vspace{-.3cm}
\end{figure}
In the two-task GridWorld problem shown in Fig.~\ref{fig:counterexample}, there are two environments with the same state and action spaces. The dynamics and reward functions, however, are different. The two actions, labeled $L$ and $R$, deterministically move the agents to the left and right, respectively, in all states in Task 1.  In Task 2, the effect of $L$ and $R$ is reversed for states $S_2$ and $S_4$: applying $L$ ($R$) in $S_2$ transitions to $S_3$ ($S_1)$, while applying $L$ ($R$) in $S_4$ transitions to $S_5$ ($S_3$).  In both environments, the agents stay in states $S_1$ and $S_5$ when they reach them.  In Task 1 there is a reward of $+1$ for reaching $S_1$ and a penalty of $-1$ for reaching $S_5$; these rewards are reversed for Task 2.


We now consider what happens when we are asked to find a single policy that maximizes the sum of the cumulative rewards of the two tasks.
It is obvious that the optimal policy for state $S_2$ and $S_4$ is to always take action $L$ in order to reach the positive reward or to stay away from the negative reward. The only state whose optimal policy remains unclear is $S_3$. With the detailed computation left to the supplementary material, we find that the optimal (stochastic) policy $\pi^*$ is 
\vspace{-5pt}
\begin{equation*}
    \pi^*(a|S_3)=\left\{\begin{array}{ll} 0.5, & a=L, \\ 0.5, & a=R,\end{array}\right.
    \vspace{-5pt}
\end{equation*}
which yields $V^{\pi^*}(S_3)=\frac{2\gamma}{2-\gamma^2}\cdot$ 

By symmetry, the two possible deterministic policies
\vspace{-5pt}
\begin{align*}
    \pi_l(a|S_3)=\left\{\begin{array}{ll} 1, & a=L \\ 0, & a=R\end{array}\right.\text{and   }\pi_r(a|S_3)=\left\{\begin{array}{ll} 0, & a=L \\ 1, & a=R\end{array}\right.
    \vspace{-5pt}
\end{align*}
produce the same value for state $S_3$. For example, if the agent takes $\pi_{l}$ as an action in $S_{3}$, it keeps moving left in task $1$ and oscillates between $S_2$ and $S_3$ in task $2$. In this case, one can show that due to the discount factor $V^{\pi_l}(S_3) =\gamma$. A similar argument holds for the case when the agent takes $\pi_{r}$ in $S_{3}$, where we have $V^{\pi_r}(S_3) = \gamma$. In both cases, $V^{\pi_l}(S_3) = V^{\pi_r}(S_3)<V^{\pi^*}(S_3)$ when $\gamma>0$, which implies that any deterministic policy is sub-optimal.

As a consequence, RL methods that implicitly rely on the existence of a deterministic optimal policy (e.g., Q learning) cannot solve this type of problem. This provides additional motivation for us to study randomized policies and take on a policy gradient approach.


\textbf{Gradient domination condition}. In single task RL, it has been shown that the objective function, despite being non-concave, satisfies a kind of ``gradient domination'' condition \citep{AgarwalKLM2019}, which implies that every stationary point is globally optimal. This is important as it guarantees that the policy gradient algorithm, by reaching a stationary point in single task RL, can find the globally optimal policy. As we show below, in the multi-task problem we cannot expect to have this condition in the general setting. The landscape of the MTRL objective is so irregular that there could exist multiple stationary points which are not global optima.
We illustrate this issue with another simple example below.\looseness=-1

Let us consider again the 2-task GridWorld problem in Fig.\ref{fig:counterexample}. Here we make a slight modification to the dynamics of the tasks. In task 1 and task 2, regardless of the action taken in state $S_2$ and $S_4$, the transition probability is
\begin{align*}
    &P_1(s|S_2)=\left\{\begin{array}{ll} 1-p, & s=S_1 \\ p, & s=S_3\end{array}\right.\\
    &P_1(s|S_4)=\left\{\begin{array}{ll} 1-p, & s=S_3 \\ p, & s=S_5\end{array}\right.\notag\\
    &P_2(s|S_2)=\left\{\begin{array}{ll} p, & s=S_1 \\ 1-p, & s=S_3\end{array}\right.\\
    &P_2(s|S_4)=\left\{\begin{array}{ll} p, & s=S_3 \\ 1-p, & s=S_5\end{array}\right.
\end{align*}
for some $0.5 < p \leq 1$.

It is obvious that the policy gradient for states $S_2$ and $S_4$ will always be zero as the value function is a constant in terms of the policy at these two states.
We only have to optimize the policy for state $S_3$.

Under the softmax parameterization, we maintain parameters $\theta_{S_3,L}$ and $\theta_{S_3,R}$ to encode the policy
\vspace{-3pt}
\begin{equation*}
    \pi_{\theta}(L|S_3)=\frac{e^{\theta_{S_3,L}}}{e^{\theta_{S_3,L}}+e^{\theta_{S_3,R}}}\text{ and }\pi_{\theta}(R|S_3)=\frac{e^{\theta_{S_3,R}}}{e^{\theta_{S_3,L}}+e^{\theta_{S_3,R}}}\cdot
\end{equation*}

We consider the case where the agents always start from state $S_3$. It can be shown that $\theta_{S_3,L}=1,\theta_{S_3,R}=\infty$ (always taking action $R$) and $\theta_{S_3,L}=\infty,\theta_{S_3,R}=1$ (always taking action $L$) are both stationary points and achieve the global maximum of the objective \eqref{sec:prob:obj}, while $\theta_{S_3,L}=1,\theta_{S_3,R}=1$ (taking action $L$ and $R$ each with probability $0.5$) is a non-globally optimal stationary point\footnote{Derivation details can be found in Section \ref{sec:appendix_examples} of the supplementary material.}. When gradient based methods are used to optimize \eqref{sec:prob:obj}, it could be trapped at the stationary points without finding the global optimality. 
In Section \ref{sec:task_conflicting} and \ref{sec:global_optimality}, we dive deeper into the problem and show that globality optimality can be reached under a restrictive structural assumption.

%% file: Decentralized_PG.tex
\section{DECENTRALIZED POLICY GRADIENT}
\label{sec:alg}
One advantage of using softmax policies is that $\theta$ is unconstrained, making \eqref{sec:prob:obj} an unconstrained optimization problem. One may attempt to apply (stochastic) gradient ascent and utilize the existing standard techniques in (stochastic) optimization to analyze its performance. However, the optimal policy, which is possibly deterministic, may be attained only by sending $\theta$ to infinity. Such an exponential scaling with the parameters makes studying the convergence of this method more challenging. To handle this challenge, a common approach in the literature is to utilize the entropy-based regularization \citep{mnih2016asynchronous,AgarwalKLM2019}. In this paper, we use the relative-entropy as a regularization for the objective in \eqref{sec:prob:obj} inspired by \citet{AgarwalKLM2019}. Specifically, the relative-entropy of $\pi_{\theta}$ is given as
\begin{align}
\text{RE}(\pi_{\theta}) &\triangleq \mathbb{E}_{s\sim\text{Unif}_{\Scal}}\left[ \text{KL}\left(U_{\Acal}, \pi_{\theta}(\cdot | s)\right)\right] \notag\\
&=-\frac{1}{|\Scal||\Acal|} \sum_{a\in\Acal} \log \pi_{\theta}(a \,|\, s)- \log |\Acal|,
\end{align}
where $U_{\Acal}$ is the uniform distribution over $\Acal$ and $\text{KL}(p,q) = \Eset_{x\sim p}[-\log(q(x)/p(x))]$. The relative-entropy regularized variant of \eqref{sec:prob:obj} is then given as 
\begin{equation}
\hspace{-0.98pt} L^{\lambda}(\theta;\vrho) = \sum_{i=1}^{N}L_{i}^{\lambda}(\theta;\rho_{i}) = \sum_{i=1}^{N}\left(V_i^{\pi_{\theta}}(\rho_i) -   \lambda\text{RE}\left(\pi_{\theta}\right)\right),
\label{sec:prob:obj_theta_reg}
\end{equation}
where $\lambda$ is a regularization parameter. Defining the discounted state visitation distribution   $d^{\pi_{\theta}}_{i}$ under a policy $\pi_{\theta}$ in the $i$-th environment
\begin{equation*}
d^{\pi_{\theta}}_{i}(s\,|\,s_{0}) \triangleq  (1-\gamma_{i})\sum_{k=0}^{\infty}\gamma_{i}^{k}P_{i}^{\pi_{\theta}}(s_{i}^{k} = s\,|\,s_{i}^{0}=s_{0}),
\end{equation*} 
and $d_{i,\rho_{i}}^{\pi_{\theta}}(s) = \Eset_{s_{0}\sim\rho_{i}}[d^{\pi_{\theta}}_{i}(s\,|\,s_{0})]$.
The gradient of $L_{i}^{\lambda}$ is \footnote{The derivation is in Section \ref{sec:policy_gradient_MA_derivation} of the supplementary material.}
\begin{align}
 \frac{\partial L_i^{\lambda}(\theta;\rho_i)}{\partial \theta_{s, a}}&=\frac{1}{1-\gamma_i} d_{i,\rho_i}^{\pi_{\theta}}(s) \pi_{\theta}(a \,|\, s) A_i^{\pi_{\theta}}(s, a)\notag\\
 &\hspace{20pt}+\frac{\lambda}{|\mathcal{S}|}\left(\frac{1}{|\mathcal{A}|}-\pi_{\theta}(a | s)\right).
 \label{SEC:PROB:OBJ_THETA_REG:GRAD}
\end{align}

Our focus now is to apply gradient ascent methods for optimizing $L^{\lambda}$ in a decentralized setting. In fact, under some proper choice of $\lambda$ the proposed algorithm can get arbitrarily close to the stationary point of problem \eqref{sec:prob:obj}.

To optimize \eqref{sec:prob:obj_theta_reg}, we use the decentralized policy gradient method formally stated in Algorithm~\ref{Alg:Distributed_PG}. In this algorithm, each agent can communicate with each other through an undirected and connected graph $\Gcal=(\Vcal,\Ecal)$, where agents $i$ and $j$ can exchange messages if and only if they are connected in $\Gcal$. We denote by $\Ncal_{i} = \{j\,:\,(i,j)\in\Ecal\}$ the set of agent $i$'s neighbors. In addition, each agent $i$ maintains its own local policy parameter $\theta_{i}$, as an estimate of the optimal $\theta^*$ of \eqref{sec:prob:obj}. Finally, $\mu_{i}$ is the local initial distribution at agent $i$, which can be chosen differently from $\rho_{i}$.



\begin{algorithm}[!ht]
\SetAlgoLined
\textbf{Initialization:} Each agent $i$ initializes $\theta_i^0\in\Rset^d$, an initial distribution $\mu_i$, and step sizes $\{\alpha^k\}_{k\in\mathbb{N}}$.\;

 \For{k=1,2,3,...}{
  Each agent $i$ simultaneously implements:
  
    \hspace{10pt}1) Exchange $\theta_i^k$ with neighbors $j\in\mathcal{N}_i$
    
    \hspace{10pt}2) Compute the gradient $g_i^k$ of $L_i^{\lambda}(\theta_i^k;\mu_i)$
    
    \hspace{10pt}3) Policy update:
        \begin{equation}
        \theta^{k+1}_{i} = \sum_{j\in\mathcal{N}_i}W_{ij}\theta_{j}^{k} + \alpha^k g_i^k.
        \label{Alg:Distributed_PG:Update}
        \end{equation}
        \vspace{-10pt}
 }
\caption{Decentralized Policy Gradient Algorithm}
\label{Alg:Distributed_PG}
\end{algorithm}

At any time $k\geq0$, agent $i$ first exchanges its iterates with its neighbors $j\in\Ncal_{i}$ and compute the gradient $g_{i}^{k}$ of $L_{i}^{\lambda}(\theta_{i}^{k};\mu_{i})$ only using information from its environment. Agent $i$ updates $\theta_{i}$ by implementing \eqref{Alg:Distributed_PG:Update}, where it takes a weighted average of $\theta_{i}^{k}$ with  $\theta_{j}^{k}$ received from its neighbors $j\in\Ncal_{i}$, following by a local gradient step. The goal of this weighted average is to achieve a consensus among the agents' parameters, i.e., $\theta_{i} = \theta_{j}$, while the local gradient steps are to push this consensus point toward the optimal $\theta^*$. Here, $W_{ij}$ is some non-negative weight, which $i$ assigns for $\theta_{j}^{k}$. The conditions on $W_{ij}$ to guarantee the convergence of Algorithm \ref{Alg:Distributed_PG} are given in the next section.

\section{CONVERGENCE ANALYSIS}
\label{sec:convergence_analysis}
In this section, our focus is to study the performance of Algorithm \ref{Alg:Distributed_PG} under the tabular setting, i.e., $\theta\in\Rset^{|\Scal||\Acal|}$. It is worth recalling that each function $V_{i}^{\pi}$ in \eqref{sec:prob:obj} is in general non-concave. To show the convergence of our algorithm, we study the case when $g_{i}$ is an exact estimate of $\nabla L_{i}^{\lambda}$ (Eq. \eqref{SEC:PROB:OBJ_THETA_REG:GRAD}), and consider the weight matrix $W$ satisfying the following assumption, which is fairly standard in the literature of decentralized consensus-based optimization \citep{zhang2019distributed,DoanMR2019_DTD(0),zeng2020finite}.      
\begin{assump}
Let $W = [W_{ij}]\in\Rset^{N\times N}$ be a doubly stochastic matrix, i.e., $\sum_{i}W_{ij} = \sum_{j}W_{ij} = 1$, with $W_{ii} > 0 $. Moreover, $W_{ij}>0$ iff $i$ and $j$ are connected, otherwise $W_{ij} = 0$. 
\label{assump:W_doublystochastic}
\end{assump}

We denote by $\sigma_2$ and $\sigma_N$ the second largest and the smallest singular values of $W$, respectively. Our first main result shows that the algorithm converges to the stationary point of \eqref{sec:prob:obj} at a rate $\Ocal(1/\sqrt{K})$, where $\mu_{i} = \rho_{i}$, for all $i$.

\begin{thm}\label{THM:MTPG_TABULAR_SOFTMAX_SADDLE}
Suppose that Assumption \ref{assump:W_doublystochastic} holds. Let $\{\theta_{i}^{k}\}$, for all $i$, be generated by Algorithm \ref{Alg:Distributed_PG}. In addition, let $\mu_{i} = \rho_{i}$, for all $i$, and the step sizes $\alpha^{k} = \alpha$ satisfying
\begin{align}
\alpha \in \Big(0,\, \frac{1+\sigma_N}{\sum_{i=1}^N \frac{16}{(1-\gamma_i)^3}+\frac{4N\lambda}{|\Scal|}}\Big).\label{thm:MTPG_tabular_softmax_saddle:stepsize}    
\end{align}
Then $\forall i$, $\theta_{i}^{k}$ satisfies
\begin{align}
\begin{aligned}
&\min_{k<K}\Big\|\frac{1}{N}\sum_{j=1}^N\nabla V_j(\theta_i^k;\rho_j)\Big\|^2\\
&\leq \Ocal\Big( \frac{1}{K\alpha}+ \frac{\alpha^2}{N(1-\sigma_{2})^2\sum_{j=1}^{N}(1-\gamma_j)^{6}}+\frac{\lambda^2}{N}\Big).
\end{aligned}
\label{thm:MTPG_tabular_softmax_saddle:rate}
\end{align}
\vspace{-5pt}
\end{thm}
For an ease of exposition, we delay the analysis of this theorem to the supplementary material, where we provide an exact formula for the right-hand side of \eqref{thm:MTPG_tabular_softmax_saddle:rate}. First, our upper bound in \eqref{thm:MTPG_tabular_softmax_saddle:rate} depends quadratically on the inverse of the spectral gap $1-\sigma_{2}$, which shows the impact of the graph $\Gcal$ on the convergence of the algorithm. 
Second, this bound states that under a constant step size the norm of the gradient converges to a ball with radius $\Ocal(\alpha)$ at a rate $\Ocal(1/\sqrt{K})$. As the step size is reduced, we get closer to a stationary point of \eqref{sec:prob:obj}.
This rate matches the one for single task RL in \citet{AgarwalKLM2019}. However, while we only show the convergence to a stationary point, a global optimality is achieved there. Below, we provide insight on this discrepancy, which has already been suggested by the example in Section \ref{SEC:CHALLENGES_MTRL}.

\begin{remark}
we note that Algorithm \ref{Alg:Distributed_PG} can be applied when the policy is represented by function approximations (e.g., neural networks). Under an function approximation, we can guarantee the convergence to a stationary point with the rate stated in Theorem \ref{THM:MTPG_TABULAR_SOFTMAX_SADDLE} under an assumption on the Lipschitzness of the function approximation, and the proof of Theorem \ref{THM:MTPG_TABULAR_SOFTMAX_SADDLE} in Section \ref{sec:proof_MTPG_tabular_softmax_saddle} of the supplementary material still goes through. 
For conciseness, we do not formally state the theorem under function approximations, but will provide numerical simulations using neural networks in Section \ref{sec:simulation}.
\end{remark}

\subsection{Task Conflicting}
\label{sec:task_conflicting}
For single task RL, policy gradient methods can return a globally optimal policy in the tabular setting \citep{AgarwalKLM2019}. A natural question to ask is whether multi-task versions of the policy gradient methods can return a globally optimal policy in MTRL. 
Unfortunately, this question is no in general, which has been shown by the example in Section \ref{SEC:CHALLENGES_MTRL}. Here we provide a more mathematical explanation of the issue. First, we consider single task RL with the softmax parameterization and relative entropy regularizer. In this setting, by using policy gradient methods one can show that $\|\nabla_{\theta}L_i^{\lambda}(\theta;\mu_i)\| \rightarrow 0$; \citep[Corollary $5.1$]{AgarwalKLM2019}, which translates to $A_i^{\pi_{\theta}}(s,a)$ decaying to zero for all $s,a$ for a proper choice of $\lambda$. The decay of the advantage function is used to show the globally optimal convergence of policy gradient methods, where the objective function satisfies \citep[Theorem $5.2$]{AgarwalKLM2019} 
\begin{equation*}
V_i^{\pi^*}\left(\rho_i\right)-V_i^{\pi_{\theta}}\left(\rho_i\right) = \frac{1}{1-\gamma}\sum_{s,a}d_{i,\rho_i}^{\pi^*}(s)\pi^*(a|s)A_i^{\pi_{\theta}}(s,a).
\end{equation*}
This condition is often referred to as the gradient domination property, an analog of the Polyak-Lojasiewicz condition commonly used to show the global optimality in nonconvex optimatization. Unfortunately, we do not have this property in the general MTRL setting. Although each component task function $V_{i}^{\pi}(\rho_i)$ satisfies the gradient domination condition, their sum may not due to what we call distribution mismatch or task conflicting. To further explain this, using a step of analysis similar to \citet[Theorem $5.2$]{AgarwalKLM2019}, we have\looseness=-1
\begin{align*}
&V(\theta^*;\vrho)-V(\theta;\vrho)\notag\\
&=\sum_{i=1}^{N}\frac{1}{1-\gamma_i} \sum_{s\in\Scal_{i}}\sum_{a\in\Acal} d_{i,\rho_i}^{\pi_{\theta^{\star}}}(s) \pi_{\theta^{\star}}(a \,|\, s) A_i^{\pi_{\theta}}(s, a)\notag\\
&= \sum_{s,a}\pi_{\theta^{\star}}(a \,|\, s)\sum_{i:s\in\Scal_i}\frac{d_{i,\rho_i}^{\pi_{\theta^{\star}}}(s) }{d_{i,\mu_i}^{\pi_{\theta}}(s) }\frac{d_{i,\mu_i}^{\pi_{\theta}}(s)}{1-\gamma_i}  A_i^{\pi_{\theta}}(s, a).
\end{align*}
To achieve the global optimality one needs to have 
\[
\vspace{-3pt}
\sum_{i:s\in\Scal_i}\frac{d_{i,\rho_i}^{\pi_{\theta^{\star}}}(s) }{d_{i,\mu_i}^{\pi_{\theta}}(s) }\frac{d_{i,\mu_i}^{\pi_{\theta}}(s)}{1-\gamma_i}  A_i^{\pi_{\theta}}(s, a) \rightarrow 0,
\]
while the policy gradient algorithm can only return (cf. \eqref{SEC:PROB:OBJ_THETA_REG:GRAD})
\[
\vspace{-3pt}
\sum_{i:s\in\Scal_i}\frac{d_{i,\mu_i}^{\pi_{\theta}}(s)}{1-\gamma_i}  A_i^{\pi_{\theta}}(s, a) \rightarrow 0.
\]
With a similar line of analysis, one can show that this issue also arises in other forms of policy, e.g., direct parameterization, and with other types of policy gradient methods, e.g., mirror descent \citep{lan2021policy}. This problem is due to the ratios $d_{i,\rho_i}^{\pi_{\theta^{\star}}}(s)/d_{i,\mu_i}^{\pi_{\theta}}(s)$ being different across the environments. We call this ratio the distribution mismatch between the environments, representing the conflict between tasks.
Because of this distribution mismatch, gradient descent approach uses biased gradients in its update. One cannot easily correct this mismatch since the optimal policy $\pi_{\theta^*}$ is unknown.


\subsection{Achieving Global Optimality}
\label{sec:global_optimality}
Despite the difficulty of the MTRL problem, we provide a sufficient condition on the structure of the MDPs, where a global optimality can be achieved by Algorithm \ref{Alg:Distributed_PG}. 


\begin{assump}\label{assump:restrictive}
Let $\pi_{\theta^*}$ be an optimal policy solving \eqref{sec:prob:obj}. Then for any $\pi_{\theta}$ and $\boldsymbol{\mu}$ we have
\begin{equation}
        \frac{d_{i,\rho_i}^{\pi_{\theta^*}}(s)}{d_{i,\mu_i}^{\pi_{\theta}}(s)}=\frac{d_{j,\rho_j}^{\pi_{\theta^*}}(s)}{d_{j,\mu_j}^{\pi_{\theta}}(s)},\,\, \forall s:s\in\Scal_i\cap\Scal_j,\; \forall i,j \in[N].  
        \label{eq:discounted_visitation_assumption}
\end{equation}
\end{assump}
We know that $d_{i,\rho_{i}}^{\pi_{\theta}}(s_{i})$ (similarly, $d_{i,\mu_{i}}^{\pi_{\theta}}(s_{i})$) is the discounted fraction of time that agent $i$ visits state $s_{i}\in\Scal_{i}$ when using $\rho_{i}$ (similarly, $\mu_{i}$) as the initial distribution. Qualitatively, this assumption can be interpreted as enforcing that the joint states between the environments are equally explored. Mathematically, this assumption guarantees the objective function \eqref{sec:prob:obj} obeys a kind of gradient domination when each function $V_i^{\pi}(\rho_i)$ satisfies this condition.  

We note that Assumption \ref{assump:restrictive} holds in the important case where the component tasks share the same state space and transition probability, but differ in their reward functions.

Under Assumption \ref{assump:restrictive}, we show that Algorithm \ref{Alg:Distributed_PG} finds the global optimality of \eqref{sec:prob:obj}. For simplicity, we assume without loss of generality that $\theta_{i}^{0} = \theta_{j}^{0}$, $\forall\,i,j$. Let $\alpha^{k} = \alpha$ satisfying 
\begin{align}
    \alpha&< \frac{1}{\sum_{i=1}^N \left(\frac{8}{(1-\gamma_i)^3}+\frac{2\lambda}{|\Scal|}\right)}\label{thm:MTPG_tabular_softmax_constantalpha:stepsizes}\\
    &\hspace{0pt}\times \min \Big\{1+\sigma_N\;;\;\frac{\lambda N(1-\sigma_2)}{4|\Scal||\Acal|\left(2N\lambda + \sum_{i=1}^N\frac{1}{(1-\gamma_i)^2}\right)}\Big\}.\notag
\end{align}
\begin{thm}\label{THM:MTPG_TABULAR_SOFTMAX_CONSTANTALPHA}
Suppose that Assumptions \ref{assump:W_doublystochastic} and \ref{assump:restrictive} hold. Given an $\epsilon >0$, let $\lambda = \epsilon\,/\, 2N \|d_{\boldsymbol{\rho}}^{\pi_{\theta^*}}/\boldsymbol{\mu}\|_{\infty}$ and $\alpha^{k}$ satisfy \eqref{thm:MTPG_tabular_softmax_constantalpha:stepsizes}. Let $\theta^*$ be a solution of \eqref{sec:prob:obj}.  Then $\forall i$, $\theta_{i}^{k}$ returned by Algorithm  \ref{Alg:Distributed_PG} satisfies
\begin{align}
\begin{aligned}
&\min_{k<K}\{V(\theta^*;\vrho)-V(\theta_i^k;\vrho)\}\leq\epsilon\\
&\text{if } K\geq \Ocal\Bigg(\frac{|S|^2|A|^2\sum_{j=1}^{N}\frac{1}{(1-\gamma_{j})^{6}}}{(1-\sigma_{2})^2\epsilon^2}\left\|\frac{d_{\boldsymbol{\rho}}^{\pi_{\theta^*}}}{\boldsymbol{\mu}}\right\|_{\infty}^{2}\Bigg),
\end{aligned}
\label{thm:MTPG_tabular_softmax_constantalpha:rates}
\end{align}
where we denote $\left\|\frac{d_{\boldsymbol{\rho}}^{\pi_{\theta^*}}}{\boldsymbol{\mu}}\right\|_{\infty} = \underset{\substack{s\in\Scal\\j:s\in\Scal_j}}{\max}\frac{d_{j,\rho_j}^{\pi_{\theta^{\star}}}(s)}{(1-\gamma_j)\mu_j(s)}\cdot$
\end{thm}

Under Assumption \ref{assump:restrictive}, Algorithm \ref{Alg:Distributed_PG} achieves a global optimality with the same rates as the ones in \citet{AgarwalKLM2019}, except for a factor $1/(1-\sigma_{2})^2$ which captures the impact of communication graph $\Gcal$.
Eq. \eqref{thm:MTPG_tabular_softmax_constantalpha:rates} also shows the impact of the initial distribution $\boldsymbol{\mu}$ on the convergence of the algorithm through the distribution mismatch coefficient. A bad choice of $\boldsymbol{\mu}$ may result in a local optimum (or stationary point) convergence by breaking Assumption \ref{assump:restrictive}, as we will illustrate by simulation in Section \ref{sec:gridworld_experiments}.

%% file: Experiments.tex
\label{sec:gridworld_experiments}
\begin{figure*}
  \includegraphics[width=\linewidth]{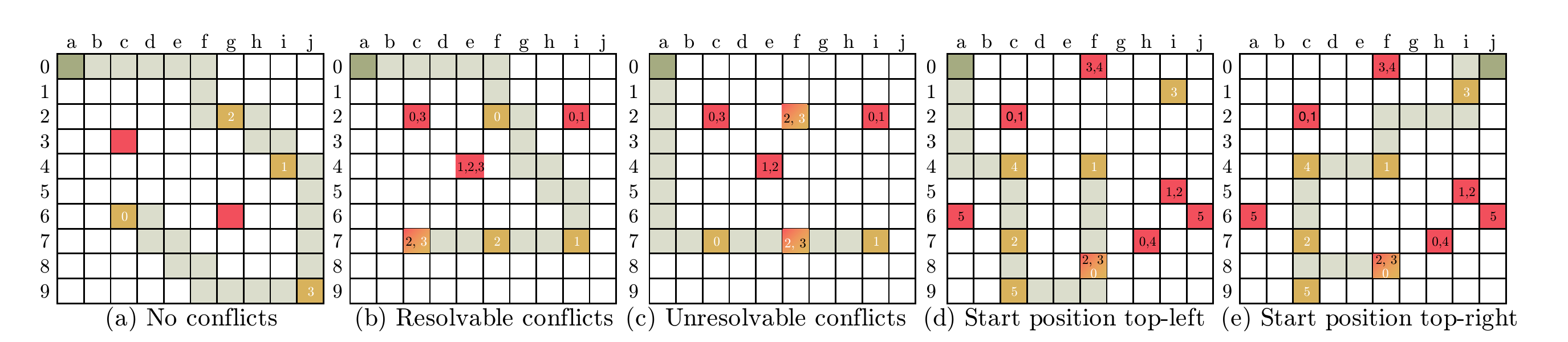}
  \vspace{-.7cm}
  \caption{Evaluate Learned Policy in Multi-task GridWorld}
  \label{fig:maze_trajectory}
  \vspace{-.3cm}
\end{figure*}

\section{EXPERIMENTAL RESULTS}
\label{sec:simulation}
\subsection{GridWorld Problems}

In this section, we evaluate the performance of our proposed algorithm on two platforms: GridWorld and drone navigation. We first verify the correctness of our theoretical results by applying the decentralized policy gradient (DCPG) algorithm for solving small-scale GridWorld problems, where each agent uses a tabular policy. We next apply the proposed method to solve the more challenging problem of large-scale drone navigation in simulated 3D environments, where the policy is approximated by neural networks. 

\textbf{General setup}. In each simulation, the agents runs a number of episodes of DCPG. In each episode, each agent computes its local gradient by using the Monte-Carlo method. Each agent then communicates with its neighbors over a fixed ring graph (i.e. agent $i$ communicates with agent $i-1$ and $i+1$ for $i=2,3,...,N-1$; agent $1$ communicates with agent $2$ and $N$; agent $N$ communicates with agent $N-1$ and $1$) and updates its iterates using \eqref{Alg:Distributed_PG:Update}. Given the communication graph $\Gcal$, we generate the weight matrix $W$ using the lazy Metropolis method \citep{olshevsky2014linear}.



To illustrate the results in Theorem \ref{THM:MTPG_TABULAR_SOFTMAX_CONSTANTALPHA}, we apply Algorithm \ref{Alg:Distributed_PG} for solving the popular GridWorld problem under tabular settings, i.e., $\theta\in\Rset^{|\Scal||\Acal|}$. This is a notable small-scale RL problem, which can be solved efficiently by using tabular methods; see for example \citet[Example $4.1$]{Sutton2018}. In this problem, the agent is placed in a grid of cells, where  each cell can be labeled either by the desired goal, an obstacle, or empty. The agent selects an action from the set of $4$ actions \{up, down, left, right\} to move to the next cell. It then receives a reward of $+1$ if it reaches the desired goal, $-1$ if it gets into an obstacle, and $0$ otherwise. The goal of the agent is to reach a desired position from an arbitrary initial location in a minimum number of steps (or maximize its cumulative rewards). 

For multi-task RL settings, we consider a number of different single GridWorld environments of size $10\times 10$, where they are different in the obstacle and goal positions. We assign each agent to one environment. In this setting, the environments have the same transition probabilities. Therefore, Assumption \ref{assump:restrictive} is satisfied when agents across environments use an identical initial state distribution. 


For solving this multi-task GridWorld, the agents implement Algorithm \ref{Alg:Distributed_PG} where the local gradients are estimated using a Monte-Carlo approach, and the states are their locations in the grid. After $1000$ training episodes, the agents agree on a unified policy, whose performance is tested in parallel in all environments. The results are presented in Fig.\ref{fig:maze_trajectory}, where we combine all the environments into one grid. In addition, yellow and red cells represent the goal and obstacle, respectively. For each environment, we terminate the test when the agent reaches the goal or hits an obstacle. The light green path is the route which the agent visits in these environments. Since we have a randomized policy, we put the path mostly followed by the agents. Fig.\ref{fig:maze_trajectory} (a)--(c) consider experiments on four environments, while (d) and (e) are on six environments. 


In Fig.\ref{fig:maze_trajectory}(a), we illustrate the performance of the policy when there is no conflict between the environments, i.e., the block of one environment is not the goal of the others and vice versa. In this case, we can see that the algorithm returns an optimal policy which finds all the goals at the environments. Next, we consider the conflict setting in Fig.\ref{fig:maze_trajectory}(b), where one obstacle of environment $2$ is the goal of environment $3$. Here, the $i$ number in white and black represents the goal and the obstacles of the $i$-th environment, respectively. Although in this case there is a conflict between the tasks, it is solvable, that is, the agents still can find the optimal policy to solve all the task. These simulations agree with our results in Theorem \ref{THM:MTPG_TABULAR_SOFTMAX_CONSTANTALPHA}, which finds the global optimality. 

We next consider an unsolvable conflict in Fig.\ref{fig:maze_trajectory}(c), where the goal of agent $2$ is the obstacle of agent $3$ and vice versa. In this case, there does not exists a policy that can always visit all goal positions without running into an obstacle. Instead, the agents need to make a compromise, where they finish three out of the four tasks.

To summarize, the experiments with no conflict and resolvable conflict have dynamics that allow the optimal value of \eqref{sec:prob:obj} to be the sum of the optimal values of the individual tasks, while the experiment with unresolvable conflict does not. Nevertheless, in all three cases, DCPG successfully finds the global optimality of the global objective function \eqref{sec:prob:obj}.\looseness=-1

Finally, we illustrate the impact of the initial conditions with the simulations in Fig.\ref{fig:maze_trajectory}(d) and (e). In (d), if the agents start from the top left corner they cannot find the optimal solution. However, when the agents start from the top right corner the algorithms return the gobal optimality as shown in (e). This empirical evidence hints that to achieve the global optimality with the DCPG algorithm, conditions on the initial state distribution like Assumption \ref{assump:restrictive} may be necessary.

\setlength{\tabcolsep}{10pt} 
\renewcommand{\arraystretch}{1} 
\begin{figure}[h]
  \hspace{-0.5cm}\includegraphics[width=1.1\columnwidth]{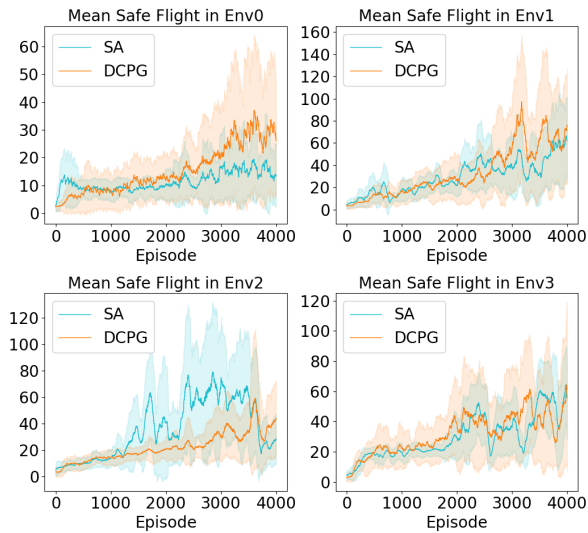}
  \caption{MSF During Training (REINFORCE)}
  \label{figure:multienv_drone_train_returns}
  \vspace{-.4cm}
\end{figure}

\begin{table*}[h!]
\centering
\caption{MSF of Learned Policy}
\vspace{-.1cm}
\begin{tabular}{cccccc}
        \toprule
        Policy (REINFORCE)  & Env0     & Env1  & Env2 & Env3 & Sum\\
        \midrule
        SA-$0$ & $15.9\pm 5.3$  & $4.5\pm 1.2$  & $4.1\pm 1.3$ & $3.6\pm 3.0$ & $28.1$  \\
        SA-$1$      & $3.0\pm 0.2$   & $55.4\pm 29.3$  & $9.7\pm 2.8$ & $8.1\pm 3.8$ & 76.2  \\
        SA-$2$     & $1.5\pm 0.5$  & $0.8\pm 0.2$  & $21.1\pm 18.3$ & $2.0\pm 0.6$ & 25.4  \\
        SA-$3$     & $2.3\pm 0.5$  & $0.8\pm 0.2$  & $8.6\pm 2.0$ & $40.1\pm 17.4$ & 51.8  \\
        DCPG (proposed) & \textbf{25.2 $\pm$ 20.1}  & \textbf{67.9 $\pm$ 35.5}  & \textbf{40.5 $\pm$ 18.0} & \textbf{61.8 $\pm$ 39.2} & \textbf{195.4}  \\

        \bottomrule
        Policy (A2C)  & Env0     & Env1  & Env2 & Env3 & Sum\\
        \midrule
        SA-$0$ & $21.8\pm 6.5$  & $7.0\pm 0.8$  & $15.1\pm 5.4$ & $14.9\pm 8.2$ & $58.8$  \\
        SA-$1$      & $1.3\pm 0.4$   & \textbf{54.1 $\pm$ 20.1}  & $2.8\pm 0.9$ & $6.4\pm 1.2$ & 59.4  \\
        SA-$2$     & $1.8\pm 0.7$  & $3.9\pm 0.3$  & $105.2\pm 38.5$ & $9.9\pm 1.3$ & 120.8  \\
        SA-$3$     & $1.1\pm 0.2$  & $1.4\pm 0.2$  & $15.8\pm 5.0$ & $78.6\pm 25.9$ & 96.9  \\
        DCPG (proposed) & \textbf{25.2 $\pm$ 7.5}  & 50.1 $\pm$ 24.6  & \textbf{165.8 $\pm$ 64.6} & \textbf{159.6 $\pm$ 61.0} & \textbf{380.7}  \\
        \bottomrule
        
        Policy (PPO)  & Env0     & Env1  & Env2 & Env3 & Sum\\
        \midrule
        SA-$0$ & \textbf{28.3 $\pm$ 15.5} & $11.2\pm6.3$ & $8.7\pm5.9$ & $13.5\pm5.7$ & 61.7 \\
        SA-$1$ & $1.1\pm0.6$ & \textbf{75.3 $\pm$ 43.2} & $1.6\pm0.4$ & $1.6\pm 0.8$ & 79.6 \\
        SA-$2$ & $2.5\pm1.8$ & $3.0\pm1.1$ & $63.2\pm 36.4$ & $15.6\pm10.6$ & 84.3\\
        SA-$3$ & $1.9\pm1.6$ & $1.2\pm0.5$ & $14.3\pm8.7$ & $139.0\pm 72.5$ & 156.4 \\
        DCPG (proposed) & $26.3\pm10.9$ & $66.7\pm30.8$ & \textbf{144.0 $\pm$ 82.4} & \textbf{195.2 $\pm$ 92.4} & \textbf{432.2}\\
        \bottomrule
        \bottomrule
        \end{tabular}
\label{table:multienv_drone}
\end{table*}

\subsection{Drone Navigation}

For the drone experiment we use PEDRA, a 3D stimulated drone navigation platform \citep{PEDRA}. In this  platform, a drone agent is equipped with a front-facing camera, and takes actions to control its flight. The reward received by the drone agent is designed to encourage the drone to stay away from obstacles. We select 4 indoor environments on the PEDRA platform (denoted as Env 0-3), which contain widely different lighting conditions, wall colors, furniture objects, and hallway structures. The performance of a policy is quantified by the mean safe flight (MSF), the average distance travelled by the agent before it collides with any obstacle. This is a standard criterion in evaluating the performance of flying autonomous vehicles \citep{sadeghi2016cad2rl}.


Instead, to evaluate the policy learned using Algorithm \ref{Alg:Distributed_PG} (DCPG), we compare it with the single agent trained independently in each environment. For brevity, we denote by SA-$i$ the single agent trained in environment $i$. We note that the SAs can be considered as the solutions to the local objective functions, while DCPG optimizes the sum of the local objective functions. Therefore, if trained to the global optimum, each SA provides an upper bound on the performance of the DCPG policy in the respective environment. The aim of the experiments is to show in practical problems, the DCPG policy often performs close to this bound.

To demonstrate the compatibility of our algorithm with a wide range of policy gradient methods, we conduct three sets of experiments, where we run Algorithm \ref{Alg:Distributed_PG} with the gradient $g_i^k$ estimated by three popular variants of policy gradient algorithms: REINFORCE, advantage actor-critic (A2C), and proximal policy optimization (PPO). In each case, a 5-layer neural network is used to approximate the policy\footnote{Implementation details
are presented in Section \ref{sec:drone_experiment_details} of the supplementary material.}. We stress that in each set of the experiments, the SAs and DCPG are trained identically, with the only difference being whether the agents communicate their policies.

In Fig.\ref{figure:multienv_drone_train_returns}, we show MSF of the DCPG and SA policies in the training phase with the REINFORCE algorithm. 

Finally, we test the policies learned by DCPG and SAs in the four environments. The results are presented in Table \ref{table:multienv_drone}. Across the three sets of experiments, we consistently see the performance difference between DCPG and the SAs. As expected, SA-$i$ only performs well in $i$-th environment but does not generalize to environment it has not seen. On the other hand, the policy returned by DCPG performs very well in all environments. Surprisingly, DCPG often performs even better than each SA-$i$ in the $i$-th environment, which we observe is due to the benefits of learning common features and representation among the agents.

%% file: Conclusion.tex
\section{CONCLUSION}
By combining consensus optimization with the policy gradient algorithm,  we propose a decentralized method that aims to learn a unified policy in MTRL problems. We theoretically show that the convergence of multi-task algorithm achieves the same convergence rate as the single task algorithm within constants depending on the connectivity of the network, and support our analysis with a series of experiments. Future directions from this work include investigating the possibility of relaxing Assumption \ref{assump:restrictive} in achieving the global optimality and improving the convergence rate.

%% file: Appendix_Examples.tex
\section{Computation Details of Examples in Section \ref{SEC:CHALLENGES_MTRL}}
\label{sec:appendix_examples}

First, we look at the first example in Section \ref{SEC:CHALLENGES_MTRL} used to illustrate that deterministic optimal policy may not exist in MTRL in general. As we discussed, it is easy to see that the optimal policy in state $S_2$ and $S_4$ is to always take action $L$ in order to reach the positive reward or to stay away from the negative reward, and all that is left to be figured out is the policy at state $S_3$.

There are 2 possible deterministic policies in state $S_3$, to always take action $L$ or to always take action $R$. First, consider one policy $\pi_{d,l}$, which is to always take $L$.

We have $V_1^{\pi_{d,l}}(S_3)=\gamma$ as the agent reaches $S_1$ in 2 steps under $\pi_{d,l}$ and claims the $+1$ reward. However, this policy produces a zero value in environment 2, $V_2^{\pi_{d}}(S_3)=0$, since an agent will move back and forth between $S_3$ and $S_4$ forever. Therefore, this deterministic policy achieves
\begin{align*}
    V_1^{\pi_{d,l}}(S_3)+V_2^{\pi_{d,l}}(S_3)= \gamma+0=\gamma. 
\end{align*}

By symmetry, the value of the policy $\pi_{d,r}$, which is to always take action $R$ in state $S_3$, is
\begin{align*}
    V_1^{\pi_{d,r}}(S_3)+V_2^{\pi_{d,r}}(S_3)= 0+\gamma=\gamma. 
\end{align*}

Now, let's consider a stochastic policy $\pi_s$, which we will show performs better than the two deterministic policies. This policy $\pi_s$ takes the same deterministic actions as $\pi_{d,l}$ and $\pi_{d,r}$ in state $S_2, S_4$, and is defined as follows for state $S_3$.
\begin{align*}
    \pi_s(a|S_3)=\left\{\begin{array}{ll} p, & a=\text{left} \\ 1-p, & a=\text{right}\end{array}\right.
\end{align*}

We compute cumulative rewards under $\pi_s$. 
\begin{align*}
    V_1^{\pi_s}(S_3) &= p\gamma+p(1-p)\gamma^3+p(1-p)^2\gamma^5+...\\
    &=p\gamma\sum_{k=0}^{\infty}\left((1-p)\gamma^2\right)^{k}\\
    &=\frac{p\gamma}{1-(1-p)\gamma^2}.
\end{align*}
Similarly, 
\begin{align*}
    V_2^{\pi_s}(S_3) &= (1-p)\gamma+(1-p)p\gamma^3+(1-p)p^2\gamma^5+...\\
    &=(1-p)\gamma\sum_{k=0}^{\infty}\left(p\gamma^2\right)^{k}\\
    &=\frac{(1-p)\gamma}{1-p\gamma^2}.
\end{align*}

Then,
\begin{align*}
    V_1^{\pi_s}(S_3)+V_2^{\pi_s}(S_3)= \frac{p\gamma}{1-(1-p)\gamma^2}+\frac{(1-p)\gamma}{1-p\gamma^2}.
\end{align*}

Taking the derivative with respect to $p$ and setting it to 0, we get
\begin{align}
    \frac{1}{(1-(1-p)\gamma^2)^2}=\frac{1}{(1-p\gamma^2)^2},
\end{align}
which leads to $p=0.5$.

The value of policy $\pi_s$ at state $S_3$ is
\begin{align*}
    V_1^{\pi_s}(S_3)+V_2^{\pi_s}(S_3)&=\frac{p\gamma}{1-(1-p)\gamma^2}+\frac{(1-p)\gamma}{1-p\gamma^2}\\
    &=\frac{2\gamma}{2-\gamma^2}.
\end{align*}

Then, we explain how the three stationary points are computed in the second example in Section \ref{SEC:CHALLENGES_MTRL}. Note that from Section \ref{sec:policy_gradient_MA_derivation}, we have that
\begin{align}
\frac{\partial}{\partial \theta_{s,a}} V_i^{\pi_{\theta}}(\rho_i) 
&=\frac{1}{1-\gamma_i}d^{\pi_{\theta}}_{i,\rho_i}(s)\pi_{\theta}(a|s)A_i^{\pi_{\theta}}(s,a)
\label{sec:appendix_examples_grad}
\end{align}

We define $D_i^{\pi_{\theta}}$ to be the $|\Scal_i|\times|\Scal_i|$ matrix where the entry $(i,j)$ is $d_i^{\pi_{\theta}}(s_i|s_j)$. It can be easily seen that 
\begin{align}
    d^{\pi_{\theta}}_{i,\rho_i}(s) = D_i^{\pi_{\theta}}\rho_i.
\end{align}

Given $P_i^{\pi_{\theta}}$ the transition probability matrix of task $i$ under policy $\pi_{\theta}$ (whose entry $(j,k)$ denotes $P_i(j\mid k)$), the matrix $D_i^{\pi_{\theta}}$ can be computed as
\begin{align}
    D_i^{\pi}=(1-\gamma P_i^{\pi})^{-1}.
    \label{eq:counterexample_DfromP}
\end{align}

Given the small scale and the known dynamics of the problem, we can also compute the value function and the Q function of the policy $\pi_{\theta}$ in the two tasks by solving the Bellman equation, from which we get $A_i^{\pi_{\theta}}(s,a)$. 
Specifically, under a policy $\pi$, the value functions associated with the first and second tasks are
\begin{align}
    V_1^{\pi} = (I-\gamma (P_1^{\pi})^{\top})^{-1}\left[\begin{array}{c}
    0 \\
    1-p \\
    0 \\
    -p \\
    0
    \end{array}\right],\quad\text{and}\quad V_2^{\pi} = (I-\gamma (P_2^{\pi})^{\top})^{-1}\left[\begin{array}{c}
    0 \\
    -p \\
    0 \\
    1-p \\
    0
    \end{array}\right].
    \label{eq:counterexample_V}
\end{align}
In addition, we can compute the Q functions
\begin{align}
    & Q_1^{\pi}(\cdot,L)=\left[0,
    \quad(1-p)+\gamma p V_1^{\pi}(S_3),
    \quad\gamma V_1^{\pi}(S_2),
    \quad\gamma (1-p) V_1^{\pi}(S_3) - p,
    \quad0\right]^{\top},\notag\\
    & Q_1^{\pi}(\cdot,R)=\left[0,
    \quad(1-p)+\gamma p V_1^{\pi}(S_3),
    \quad\gamma V_1^{\pi}(S_4),
    \quad\gamma (1-p) V_1^{\pi}(S_3) - p,
    \quad0\right]^{\top},\notag\\
    & Q_2^{\pi}(\cdot,L)=\left[0,
    \quad\gamma(1-p)V_2^{\pi}(S_3)-p,
    \quad\gamma V_2^{\pi}(S_2),
    \quad\gamma p V_2^{\pi}(S_3) +(1 - p),
    \quad0\right]^{\top},\notag\\
    & Q_2^{\pi}(\cdot,R)=\left[0,
    \quad\gamma(1-p)V_2^{\pi}(S_3)-p,
    \quad\gamma V_2^{\pi}(S_4),
    \quad\gamma p V_2^{\pi}(S_3) +(1 - p),
    \quad0\right]^{\top},
    \label{eq:counterexample_Q}
\end{align}
from which the advantage function can be easily computed by taking the difference between the Q functions and the value functions.
We also know $\pi_{\theta}(s,a)$ of the policy for which we would like to evaluate the gradient. Therefore, we can compute all the quantities in the gradient expression \eqref{sec:appendix_examples_grad}. Now we go through all three parameterizations and calculate the gradient and the cumulative return.

We first consider the policy $\pi_1$ under the parameterization $\theta_{S_3,L}=1,\theta_{S_3,R}=\infty$, which implies $\pi_1(L\mid S_3)=0$ and $\pi_1(R\mid S_3)=1$. First, we can easily see that the transition probability matrices are
\begin{align*}
    P_1^{\pi_1}=\left[\begin{array}{ccccc}
    1 & 1-p & 0 & 0 & 0 \\
    0 & 0 & 0 & 0 & 0 \\
    0 & p & 0 & 1-p & 0 \\
    0 & 0 & 1 & 0 & 0 \\
    0 & 0 & 0 & p & 1
    \end{array}\right],\quad\text{and}\quad
    P_2^{\pi_1}=\left[\begin{array}{ccccc}
    1 & p & 0 & 0 & 0 \\
    0 & 0 & 0 & 0 & 0 \\
    0 & 1-p & 0 & p & 0 \\
    0 & 0 & 1 & 0 & 0 \\
    0 & 0 & 0 & 1-p & 1
    \end{array}\right].
\end{align*}

Computing $D_i^{\pi_1}$ according to \eqref{eq:counterexample_DfromP} using Gaussian elimination, we can derive
\begin{align*}
    D_1^{\pi_1}=\left[\begin{array}{ccccc}
    1 & \gamma(1 - p) & 0 & 0 & 0 \\
    0 & 1-\gamma & 0 & 0 & 0 \\
    0 & \frac{\gamma p (1 - \gamma)}{(\gamma^2 p - \gamma^2 + 1)} & \frac{1-\gamma}{\gamma^2 p - \gamma^2 + 1} & \frac{\gamma(1-\gamma)(1-p)}{\gamma^2 p - \gamma^2 + 1} & 0 \\
    0 & \frac{\gamma^2 p (1 - \gamma)}{(\gamma^2 p - \gamma^2 + 1)} & \frac{\gamma(1-\gamma)}{\gamma^2 p - \gamma^2 + 1} & \frac{1-\gamma}{\gamma^2 p - \gamma^2 + 1} & 0 \\
    0 & \frac{\gamma^3 p^2}{(\gamma^2 p - \gamma^2 + 1)} & \frac{\gamma^2 p}{\gamma^2 p - \gamma^2 + 1} & \frac{\gamma p}{\gamma^2 p - \gamma^2 + 1} & 1
    \end{array}\right],\quad\text{and}\quad
    D_2^{\pi_1}=\left[\begin{array}{ccccc}
    1 & \gamma p & 0 & 0 & 0 \\
    0 & 1-\gamma & 0 & 0 & 0 \\
    0 & \frac{\gamma(1-\gamma)(1-p)}{1-\gamma^2 p} & \frac{1-\gamma}{1-\gamma^2 p} & \frac{\gamma(1-\gamma)p}{1-\gamma^2 p} & 0 \\
    0 & \frac{\gamma^2(1-\gamma)(1-p)}{1-\gamma^2 p} & \frac{\gamma(1-\gamma)}{1-\gamma^2 p} & \frac{1-\gamma}{1-\gamma^2 p} & 0 \\
    0 & \frac{\gamma^3(1-p)^2}{1-\gamma^2 p} & \frac{\gamma^2(1-p)}{1-\gamma^2 p} & \frac{\gamma(1-p)}{1-\gamma^2 p} & 1
    \end{array}\right].
\end{align*}

As we explained in \eqref{eq:counterexample_V} and \eqref{eq:counterexample_Q}, we can compute the advantage functions
\begin{align*}
    & A_1^{\pi_1}(\cdot,L)=\left[0,
    0,
    \quad\frac{\gamma(-\gamma^2 p^2  + (1-p)(\gamma^2 p - \gamma^2 + 1)+ p)}{\gamma^2 p - \gamma^2 + 1},
    0,
    \quad 0\right]^{\top},\notag\\
    & A_1^{\pi_1}(\cdot,R)=\left[0,
    \quad 0,
    \quad 0,
    \quad 0,
    \quad 0\right]^{\top},\notag\\
    & A_2^{\pi_1}(\cdot,L)=\left[0,
    \quad 0,
    \quad \frac{\gamma(\gamma^2(1-p)^2 + p(\gamma^2 p - 1) - (1- p))}{1-\gamma^2 p},
    \quad 0,
    \quad 0\right]^{\top},\notag\\
    & A_2^{\pi_1}(\cdot,R)=\left[0,
    \quad 0,
    \quad 0,
    \quad 0,
    \quad 0\right]^{\top}.
\end{align*}

Recall \eqref{sec:appendix_examples_grad}. We have
\begin{align*}
    \frac{\partial}{\partial \theta_{S_3,L}} (V_1^{\pi_1}(\rho_1) + V_2^{\pi_1}(\rho_2) )
    &=\frac{1}{1-\gamma}d^{\pi_1}_{1,\rho_1}(S_3)\pi_1(L|S_3)A_1^{\pi_1}(S_3,L)+\frac{1}{1-\gamma}d^{\pi_1}_{2,\rho_2}(S_3)\pi_1(L|S_3)A_2^{\pi_1}(S_3,L)=0,
\end{align*}
since $\pi_1(L\mid S_3)=0$. In addition, we have
\begin{align*}
    \frac{\partial}{\partial \theta_{S_3,R}} (V_1^{\pi_1}(\rho_1) + V_2^{\pi_1}(\rho_2) )
    &=\frac{1}{1-\gamma}d^{\pi_1}_{1,\rho_1}(S_3)\pi_1(R|S_3)A_1^{\pi_1}(S_3,R)+\frac{1}{1-\gamma}d^{\pi_1}_{2,\rho_2}(S_3)\pi_1(R|S_3)A_2^{\pi_1}(S_3,R)=0,
\end{align*}
since $A_1^{\pi_1}(S_3,R)=A_2^{\pi_1}(S_3,R)=0$. The cumulative return under this policy is
\begin{align*}
    V_1^{\pi_1}(\rho_1)+V_2^{\pi_1}(\rho_2) = V_1^{\pi_1}(S_3)+V_2^{\pi_1}(S_3)=\frac{\gamma(-2\gamma^2 p + \gamma^2 + 2p - 1)}{\gamma^4 p^2 - \gamma^4 p + \gamma^2 - 1}.
\end{align*}

By symmetry, the second policy $\pi_2$ under parameterization $\theta_{S_3,L}=\infty,\theta_{S_3,R}=1$ is also a stationary point and has a cumulative return
\begin{align*}
    V_1^{\pi_2}(\rho_1)+V_2^{\pi_2}(\rho_2) = \frac{\gamma(-2\gamma^2 p + \gamma^2 + 2p - 1)}{\gamma^4 p^2 - \gamma^4 p + \gamma^2 - 1}.
\end{align*}

Finally, we look at the policy $\pi_3$ under parameterization $\theta_{S_3,L}=1,\theta_{S_3,R}=1$, which implies $\pi_3(L\mid S_3)=\pi_3(R\mid S_3)=0.5$. We can see that the transition probability matrices are
\begin{align*}
    P_1^{\pi_3}=\left[\begin{array}{ccccc}
    1 & 1-p & 0 & 0 & 0 \\
    0 & 0 & 0.5 & 0 & 0 \\
    0 & p & 0 & 1-p & 0 \\
    0 & 0 & 0.5 & 0 & 0 \\
    0 & 0 & 0 & p & 1
    \end{array}\right],\quad\text{and}\quad
    P_2^{\pi_3}=\left[\begin{array}{ccccc}
    1 & p & 0 & 0 & 0 \\
    0 & 0 & 0.5 & 0 & 0 \\
    0 & 1-p & 0 & p & 0 \\
    0 & 0 & 0.5 & 0 & 0 \\
    0 & 0 & 0 & 1-p & 1
    \end{array}\right].
\end{align*}

Computing $D_i^{\pi_3}$ according to \eqref{eq:counterexample_DfromP} using Gaussian elimination, we can derive
\begin{align*}
    &D_1^{\pi_3}=\left[\begin{array}{ccccc}
    1 & \frac{\gamma(-\gamma^2 p^2 +2\gamma^2 p -\gamma^2 -2p +2)}{2-\gamma^2} & \frac{\gamma^2(1-p)}{2-\gamma^2} & \frac{\gamma^3(1-p)^2}{2-\gamma^2} & 0 \\
    0 & \frac{(1-\gamma)(\gamma^2 p -\gamma^2 + 2)}{2-\gamma^2} & \frac{\gamma(1-\gamma)}{2-\gamma^2} & \frac{\gamma^2(1-\gamma)(1-p)}{2-\gamma^2} & 0 \\
    0 & \frac{2\gamma(1-\gamma)(1-p)}{2-\gamma^2} & \frac{2(1-\gamma)}{2-\gamma^2} & \frac{2\gamma(1-\gamma)(1-p)}{2-\gamma^2} & 0 \\
    0 & \frac{\gamma^2(1-\gamma)p}{2-\gamma^2} & \frac{\gamma(1-\gamma)}{2-\gamma^2} & \frac{(1-\gamma)(2-\gamma^2 p)}{2-\gamma^2} & 0 \\
    0 & \frac{\gamma^3 p^2}{2-\gamma^2} & \frac{\gamma^2 p}{2-\gamma^2} & \frac{\gamma p(2-\gamma^2 p)}{2-\gamma^2} & 1
    \end{array}\right],\\
    &\text{and}\quad
    D_2^{\pi_3}=\left[\begin{array}{ccccc}
    1 & \frac{\gamma p(2-\gamma^2 p)}{2-\gamma^2} & \frac{\gamma^2 p}{2-\gamma^2} & \frac{\gamma^2(1-\gamma)p}{2-\gamma^2} & 0 \\
    0 & \frac{(1-\gamma)(2-\gamma^2 p)}{2-\gamma^2} & \frac{\gamma(1-\gamma)}{2-\gamma^2} & \frac{\gamma^2(1-\gamma)p}{2-\gamma^2} & 0 \\
    0 & \frac{2\gamma(1-\gamma)(1-p)}{2-\gamma^2} & \frac{2(1-\gamma)}{2-\gamma^2} & \frac{2\gamma(1-\gamma)p}{2-\gamma^2} & 0 \\
    0 & \frac{\gamma^2(1-\gamma)(1-p)}{2-\gamma^2} & \frac{\gamma(1-\gamma)}{2-\gamma^2} & \frac{(1-\gamma)(2-\gamma^2+\gamma^2 p)}{2-\gamma^2} & 0 \\
    0 & \frac{\gamma^3(1-p)^2}{2-\gamma^2} & \frac{\gamma^2(1-p)}{2-\gamma^2} & \frac{\gamma(2-\gamma^2 p^2+2\gamma^2 p-\gamma^2-2p)}{2-\gamma^2} & 1
    \end{array}\right].
\end{align*}

As we explained in \eqref{eq:counterexample_V} and \eqref{eq:counterexample_Q}, we can compute the advantage functions
\begin{align*}
    & A_1^{\pi_3}(\cdot,L)=\left[0,
    0,
    \quad\frac{\gamma(-2\gamma^2 p^2 +2\gamma^2 p -\gamma^2 +1)}{2-\gamma^2},
    0,
    \quad 0\right]^{\top},\notag\\
    & A_1^{\pi_3}(\cdot,R)=\left[0,
    \quad 0,
    \quad \frac{\gamma(2\gamma^2 p^2 -2\gamma^2 p +\gamma^2 -1)}{2-\gamma^2},
    \quad 0,
    \quad 0\right]^{\top},\notag\\
    & A_2^{\pi_3}(\cdot,L)=\left[0,
    \quad 0,
    \quad \frac{\gamma(2\gamma^2 p^2 -2\gamma^2 p +\gamma^2 -1)}{2-\gamma^2},
    \quad 0,
    \quad 0\right]^{\top},\notag\\
    & A_2^{\pi_3}(\cdot,R)=\left[0,
    \quad 0,
    \quad \frac{\gamma(-2\gamma^2 p^2 +2\gamma^2 p -\gamma^2 +1)}{2-\gamma^2},
    \quad 0,
    \quad 0\right]^{\top},
\end{align*}

From \eqref{sec:appendix_examples_grad}, we have
\begin{align*}
    \frac{\partial}{\partial \theta_{S_3,L}} (V_1^{\pi_3}(\rho_1) + V_2^{\pi_3}(\rho_2) )
    &=\frac{1}{1-\gamma}\pi_3(L|S_3)\left(d^{\pi_3}_{1,\rho_1}(S_3)A_1^{\pi_3}(S_3,L)+d^{\pi_3}_{2,\rho_2}(S_3)A_2^{\pi_3}(S_3,L)\right)\\
    &=\frac{0.5}{1-\gamma}\left(\frac{2(1-\gamma)}{2-\gamma^{2}}\cdot\frac{\gamma(-2\gamma^2 p^2 +2\gamma^2 p -\gamma^2 +1)}{2-\gamma^2}+\frac{2(1-\gamma)}{2-\gamma^{2}}\cdot\frac{\gamma(2\gamma^2 p^2 -2\gamma^2 p +\gamma^2 -1)}{2-\gamma^2}\right)\\
    &=0.
\end{align*}

Similarly, one can show
\begin{align*}
    \frac{\partial}{\partial \theta_{S_3,R}} (V_1^{\pi_3}(\rho_1) + V_2^{\pi_3}(\rho_2) )=0.
\end{align*}

The cumulative return under this policy is
\begin{align*}
    V_1^{\pi_3}(\rho_1)+V_2^{\pi_3}(\rho_2) = V_1^{\pi_1}(S_3)+V_2^{\pi_1}(S_3)=\frac{\gamma(2-4p)}{2-\gamma^2}.
\end{align*}

For computational simplicity, we choose $\gamma=\sqrt{0.5}$. Then,
\begin{align*}
    &V_1^{\pi_1}(\rho_1)+V_2^{\pi_1}(\rho_2)=\frac{\gamma(-2\gamma^2 p + \gamma^2 + 2p - 1)}{\gamma^4 p^2 - \gamma^4 p + \gamma^2 - 1}=\frac{2p-1}{8\sqrt{2}(p-2)(p+1)},\\
    \text{and}\quad&V_1^{\pi_3}(\rho_1)+V_2^{\pi_3}(\rho_2) = V_1^{\pi_1}(S_3)+V_2^{\pi_1}(S_3)=\frac{\gamma(2-4p)}{2-\gamma^2}=\frac{4-8p}{3}.
\end{align*}

If $p>0.5$,
\begin{align*}
    V_1^{\pi_1}(\rho_1)+V_2^{\pi_1}(\rho_2)=V_1^{\pi_2}(\rho_1)+V_2^{\pi_2}(\rho_2)=\frac{2p-1}{8\sqrt{2}(p-2)(p+1)}>\frac{4-8p}{3}=V_1^{\pi_3}(\rho_1)+V_2^{\pi_3}(\rho_2).
\end{align*}


%% file: Appendix_Proof.tex
\section{Proofs}

In this appendix, we provide complete analysis for the results stated in the main paper. We first introduce the following notations used throughout this appendix. 

\begin{align}
\begin{aligned}
&\vtheta \triangleq\left[\theta_1^T,\theta_2^T,...,\theta_N^T\right]^T\in\Rset^{N |\Scal||\Acal|},\qquad  \mV(\vtheta;\vrho) \triangleq\left(\begin{array}{c}V_1^{\pi_{\theta_1}}(\rho_1) \\ V_2^{\pi_{\theta_2}}(\rho_2) \\ \vdots \\ V_N^{\pi_{\theta_N}}(\rho_N)\end{array}\right)\in\Rset^{N},\\
&\vrho=[\rho_1^T,\rho_2^T,...,\rho_N^T]^T,\quad \vmu=[\mu_1^T,\mu_2^T,...,\mu_N^T]^T,\quad \overline{\nabla\mV}(\vtheta;\vrho)=\frac{1}{N}\sum_{i=1}^N \nabla_{\theta_{i}} V_i^{\pi_{\theta_i}}(\rho_i).
\end{aligned}
\label{eq:vector_func_param_def}
\end{align}

\subsection{Proof of Theorem \ref{THM:MTPG_TABULAR_SOFTMAX_SADDLE}}
\label{sec:proof_MTPG_tabular_softmax_saddle}
Define $D=2N\lambda + \sum_{i=1}^N\frac{1}{(1-\gamma_i)^2}$. In the proof, we will need the following lemmas.
\begin{lemma}
For all $k$ and $\vmu$, $||\nabla\mL^{\lambda}(\vtheta^k;\vmu)||\leq D$.
\end{lemma}
\begin{proof}
By Eq. \eqref{eq:softmax_grad_inproof},
\begin{align*}
||\nabla L_i^{\lambda}(\theta_i^k;\mu_i)||&\leq\sum_{s,a}\left|\frac{\partial L_i^{\lambda}(\theta^k_{i};\mu_i)}{\partial \theta^k_{i\,s, a}}\right|\\
&\leq\sum_{s,a}\left|\frac{1}{1-\gamma_i} d_{\mu_j}^{\pi_{\theta}}(s) \pi_{\theta}(a \,|\, s) A_i^{\pi_{\theta}}(s, a)+\frac{\lambda}{|\mathcal{S}|}\left(\frac{1}{|\mathcal{A}|}-\pi_{\theta}(a | s)\right)\right|\\
&\leq \sum_{s,a}\frac{d_{\mu_j}^{\pi_{\theta}}(s) \pi_{\theta}(a \,|\, s)}{1-\gamma_i}\frac{1}{1-\gamma_i}+\sum_{s,a}\frac{\lambda}{|\Scal||\Acal|}+\sum_{s,a}\frac{\lambda}{|\Scal|}\pi_{\theta}(a\,|\,s)\\
&\leq\frac{1}{(1-\gamma_i)^2}+2\lambda,
\end{align*}
where the second last inequality uses \eqref{eq:bounded_V_A}. Using triangular inequality,
\begin{align}
    ||\nabla\mL^{\lambda}(\vtheta^k;\vmu)||\leq\sum_{i=1}^N||\nabla L_i^{\lambda}(\theta_i^k;\mu_i)||\leq 2N\lambda + \sum_{i=1}^N\frac{1}{(1-\gamma_i)^2}.
\end{align}
\end{proof}

\begin{lemma}
Let $\bar{\theta}^k=\frac{1}{N}\sum_{i=1}^N \theta_i^k$. If each agent starts with the same initialization, i.e. $\theta_1^0=\theta_2^0=...=\theta_N^0$, then
\begin{align}
    ||\theta_i^k-\bar{\theta}^k||\leq\frac{\alpha D}{1-\sigma_2}, \quad \forall i,k.
\end{align}
\label{lem:bounded_deviation_from_mean}
\end{lemma}
This is a standard result whose proof can be found in the existing literature, such as \citet{yuan2016convergence}. 

We made the assumption in Theorem \ref{THM:MTPG_TABULAR_SOFTMAX_SADDLE} that the agents start with the same initialization. We denote $\theta^0 = \theta_i^0,\,\forall i$.

We define the Lyapunov function
\begin{align}
    \vxi_{\alpha,\lambda}(\vtheta;\vmu) \triangleq -\vct{1}^{T} \mL^{\lambda}(\vtheta;\vmu)+\frac{1}{2 \alpha}||\vtheta||_{I-W}^{2},
    \label{eq:lyapunov_def}
\end{align}
where $||\vtheta||_{I-W}^2\triangleq\vtheta^T((I-W)\otimes I)\vtheta$.

Note that the sequence $\{\vtheta^k\}$ generated by the distributed policy gradient algorithm is the same as the sequence generated by applying gradient descent on $\xi_{\alpha,\lambda}(\vtheta)$, if both algorithms use fixed step size $\alpha$. This can be observed by re-writing the update equation \eqref{Alg:Distributed_PG:Update}.
\begin{align}
    \vtheta^{k+1} &= (W\otimes I)\vtheta^{k}+\alpha\nabla\mL^{\lambda}(\vtheta^{k};\vmu)\notag\\
    &= \vtheta^{k}+\alpha\nabla\mL^{\lambda}(\vtheta^{k};\vmu)-((I-W)\otimes I)\vtheta^{k}\notag\\
    &= \vtheta^{k}-\alpha(-\nabla\mL^{\lambda}(\vtheta^{k};\vmu)+\frac{1}{\alpha}((I-W)\otimes I)\vtheta^{k})\notag\\
    &= \vtheta^{k}-\alpha\nabla\vxi_{\alpha,\lambda}(\vtheta^k;\vmu) \label{eq:lyapunov_update_softmax_constantalpha}
\end{align}

We have to establish the smoothness constant of $\vxi_{\alpha,\lambda}(\vtheta;\vmu)$. Combining Lemma \ref{lem:V_softmax_Hessian_Lipschitz} and Lemma \ref{lem:RE_softmax_Hessian_Lipschitz}, $L_i^{\lambda}(\theta_i)$ is $\beta_i^{\lambda}$-smooth with
\begin{align}
    \beta_i^{\lambda}=\frac{8}{(1-\gamma_i)^3} + \frac{2\lambda}{|\Scal|},
\end{align}
which implies $\sum_{i=1}^N L_i^{\lambda}(\theta_i)$ is $\beta^{\lambda}$-smooth, where 
\begin{align}
    &\beta^{\lambda} = \sum_{i=1}^N\left(\frac{8}{(1-\gamma_i)^3}+\frac{2\lambda}{|\Scal|}\right).
    \label{eq:smoothness_softmax}
\end{align}

In addition, we know $\vxi_{\alpha,\lambda}(\vtheta;\vmu)$ is $\beta^{\vxi_{\alpha,\lambda}}$-smooth, with
\begin{align}
    \beta^{\vxi_{\alpha,\lambda}}=\beta^{\lambda}+\frac{1}{\alpha}\sigma_{\max}(I-W)=\beta^{\lambda}+\alpha^{-1}(1-\sigma_N).
    \label{eq:smoothness_softmax_lyapunov}
\end{align}

By the $\beta^{\vxi_{\alpha,\lambda}}$-smoothness of $\vxi_{\alpha,\lambda}(\vtheta)$, we have
\begin{align}
    \vxi_{\alpha,\lambda}(\vtheta^{k+1};\vmu)&\leq\vxi_{\alpha,\lambda}(\vtheta^k;\vmu)+\langle\nabla \vxi_{\alpha,\lambda}(\vtheta^{k};\vmu), \vtheta^{k+1}-\vtheta^{k}\rangle+\frac{\beta^{\xi_{\alpha,\lambda}}}{2}||\vtheta^{k+1}-\vtheta^{k}||^2 \notag\\
    &=\vxi_{\alpha,\lambda}(\vtheta^k;\vmu)+\langle-\frac{\vtheta_{k+1}-\vtheta_{k}}{\alpha}, \vtheta^{k+1}-\vtheta^{k}\rangle+\frac{\beta^{\vxi_{\alpha,\lambda}}}{2}||\vtheta^{k+1}-\vtheta^{k}||^2 \notag\\
    &=\vxi_{\alpha,\lambda}(\vtheta^k;\vmu)+(\frac{\beta^{\vxi_{\alpha,\lambda}}}{2}-\frac{1}{\alpha})||\vtheta^{k+1}-\vtheta^{k}||^2 \notag\\
    &=\vxi_{\alpha,\lambda}(\vtheta^k;\vmu)-\frac{1}{2}(\alpha^{-1}(1+\sigma_N)-\beta^{\lambda})||\vtheta^{k+1}-\vtheta^{k}||^2
\end{align}




Since $\alpha\leq\frac{1+\sigma_N}{2\sum_{i=1}^N (\frac{8}{(1-\gamma_i)^3}+\frac{2\lambda}{|\Scal|})}=\frac{1+\sigma_N}{2\beta^{\lambda}}$, we know $\frac{1}{2}(\alpha^{-1}(1+\sigma_N)-\beta^{\lambda})\geq 0,\,\forall k$. This implies $\vxi_{\alpha,\lambda}(\vtheta^{k};\vmu)$ is a non-increasing sequence. Let $\tilde{\vtheta}=\min_{\vtheta}\xi_{\alpha,\lambda}(\vtheta;\vmu)$. We have
\begin{align}
    \sum_{k=0}^{K-1}||\vtheta^{k+1}-\vtheta^{k}||^2 &\leq \sum_{k=0}^{K-1}2(\alpha^{-1}(1+\sigma_N)-\beta^{\lambda})^{-1}(\vxi_{\alpha,\lambda}(\vtheta^{k};\vmu)-\vxi_{\alpha,\lambda}(\vtheta^{k+1};\vmu))\notag\\
    &=c_1(\vxi_{\alpha,\lambda}(\vtheta^{0};\vmu)-\vxi_{\alpha,\lambda}(\vtheta^{K-1};\vmu))\notag\\
    &\leq c_1(\vxi_{\alpha,\lambda}(\vtheta^{0};\vmu)-\vxi_{\alpha,\lambda}(\tilde{\vtheta};\vmu)),
\end{align}
where we define $c_1=2(\alpha^{-1}(1+\sigma_N)-\beta^{\lambda})^{-1}$. 

This implies 
\begin{align}
    \min_{k<K}||\vtheta^{k+1}-\vtheta^{k}||^2\leq\frac{c_1}{K}(\vxi_{\alpha,\lambda}(\vtheta^{0};\vmu)-\vxi_{\alpha,\lambda}(\tilde{\vtheta};\vmu)).
\end{align}

From Eq. \eqref{eq:lyapunov_update_softmax_constantalpha}, $||\alpha\nabla\vxi_{\alpha,\lambda}(\vtheta^k;\vmu)||^2=||\vtheta^{k+1}-\vtheta^{k}||^2$. Thus,
\begin{align}
    \min_{k<K}||\nabla\vxi_{\alpha,\lambda}(\vtheta^k;\vmu)||^2=
    \frac{1}{\alpha^2}\min_{k<K}||\vtheta^{k+1}-\vtheta^{k}||^2\leq\frac{c_1}{K\alpha^2}(\vxi_{\alpha,\lambda}(\vtheta^{0};\vmu)-\vxi_{\alpha,\lambda}(\tilde{\vtheta};\vmu)).
    \label{eq:softmax_bounded_xi}
\end{align}

Taking derivative of Eq. \eqref{eq:lyapunov_def},
\begin{align}
    \nabla\vxi_{\alpha,\lambda}(\vtheta;\vmu)=-\nabla \mL^{\lambda}(\vtheta;\vmu)+\frac{1}{\alpha}((I-W)\otimes I)\vtheta,
\end{align}

Observe that $\vct{1}^T (I-W)=\vct{0}$ due to the double stochasticity of $W$, which leads to
\begin{align*}
    \overline{\nabla\vxi}_{\alpha,\lambda}(\vtheta;\vmu)&=-\overline{\nabla\mL}^{\lambda}(\vtheta;\vmu)+\frac{1}{N\alpha}(\vct{1}^T(I-W)\otimes I)\vtheta\\
    &=-\overline{\nabla\mL}^{\lambda}(\vtheta;\vmu).
\end{align*}

Now we can bound the gradient $\overline{\nabla \mL}^{\lambda}(\vtheta^k;\vmu)$. 
\begin{align}
    \min_{k<K}||\overline{\nabla \mL}^{\lambda}(\vtheta^k;\vmu)||^2&=\min_{k<K}||\overline{\nabla\vxi}_{\alpha,\lambda}(\vtheta^k;\vmu)||^2\notag\\
    &\leq \min_{k<K}||\nabla\vxi_{\alpha,\lambda}(\vtheta^k;\vmu)||^2\notag\\
    &\leq\frac{c_1}{K\alpha^2}(\vxi_{\alpha,\lambda}(\vtheta^{0};\vmu)-\vxi_{\alpha,\lambda}(\tilde{\vtheta};\vmu))\notag\\
    &=\frac{c_1}{K\alpha^2}(-\sum_{i=1}^N L_i^{\lambda}(\vtheta^0;\mu_i)+\frac{1}{2\alpha}||\vtheta^0||_{I-W}^2+\sum_{i=1}^N L_i^{\lambda}(\tilde{\vtheta};\mu_i)-\frac{1}{2\alpha}||\tilde{\vtheta}||_{I-W}^2)\notag\\
    &\leq \frac{c_1}{K\alpha^2}\sum_{i=1}^N (L_i^{\lambda}(\tilde{\vtheta};\mu_i)-L_i^{\lambda}(\vtheta^0;\mu_i))\notag\\
    &\leq\frac{c_1}{K\alpha^2}\sum_{i=1}^N (V_i^{\pi_{\tilde{\vtheta}_i}}(\mu_i)-V_i^{\pi_{\theta_i^0}}(\mu_i)+\lambda\text{RE}(\pi_{\theta_i^0}))\notag\\
    &\leq\frac{c_1}{K\alpha^2}\sum_{i=1}^N (\frac{1}{1-\gamma_i}+\lambda\text{RE}(\pi_{\theta^0})).
    \label{eq:graident_convergence_L}
\end{align}
The third line comes from \eqref{eq:softmax_bounded_xi}. The fifth line uses our assumption that all agents start with the same parameter initialization, making $||\vtheta^0||_{I-W}^2=0$.The second last inequality is from the fact that relative entropy is non-negative. The last inequality comes from the bounded value function \eqref{eq:bounded_V_A}.

Using the definition of $\mL^{\lambda}(\vtheta^k;\vmu)$ in \eqref{sec:prob:obj_theta_reg}, we have
\begin{align}
    \min_{k<K}||\overline{\nabla V}(\vtheta^k;\vmu)||^2&=\min_{k<K}||\overline{\nabla \mL}^{\lambda}(\vtheta^k;\vmu)+\frac{\lambda}{N}\sum_{i=1}^{N}\nabla\text{RE}(\pi_{\theta_i^k})||^2\notag\\
    &\leq 2\min_{k<K}||\overline{\nabla \mL}^{\lambda}(\vtheta^k;\vmu)||^2+\frac{2}{N}\sum_{i=1}^{N}||\nabla\lambda\text{RE}(\pi_{\theta_i^k})||^2.
\end{align}

The second term uses the smoothness of the regularizer, which we establish in Lemma \ref{lem:RE_softmax_Hessian_Lipschitz}. The first term is bounded in \eqref{eq:graident_convergence_L}. Therefore,
\begin{align}
     \min_{k<K}||\overline{\nabla V}(\vtheta^k;\vmu)||^2&\leq 2\min_{k<K}||\overline{\nabla \mL}^{\lambda}(\vtheta^k;\vmu)||^2+\frac{2}{N}\sum_{i=1}^{N}||\nabla\lambda\text{RE}(\pi_{\theta_i^k})||^2\notag\\
     &\leq \frac{2c_1}{K\alpha^2}\sum_{i=1}^N (\frac{1}{1-\gamma_i}+\lambda\text{RE}(\pi_{\theta^0}))+\frac{2}{N}\left(\frac{\lambda}{\sqrt{|\Acal|}}+\lambda\right)^2\notag\\
     &\leq \frac{2c_1}{K\alpha^2}\sum_{i=1}^N (\frac{1}{1-\gamma_i}+\lambda\text{RE}(\pi_{\theta^0}))+\frac{8\lambda^2}{N}
     \label{eq:Vbar_bound_thm2}
\end{align}

Using the smoothness of $V_i$, which we show in Lemma \ref{lem:V_softmax_Hessian_Lipschitz}, we have
\begin{align}
     \min_{k<K}||\frac{1}{N}\sum_{j=1}^N\nabla V_j(\theta_i^k;\mu_j)||^2&= \min_{k<K}||\frac{1}{N}\sum_{j=1}^N\nabla V_j(\theta_j^k;\mu_j)-\left(\nabla V_j(\theta_j^k;\mu_j)-\nabla V_j(\theta_i^k;\mu_j)\right)||^2\notag\\
     &\leq \min_{k<K}2||\frac{1}{N}\sum_{j=1}^N\nabla V_j(\theta_j^k;\mu_j)||^2\notag\\
     &\hspace{50pt}+\frac{2}{N}\sum_{j=1}^{N}||\nabla V_j(\theta_j^k;\mu_j)-\nabla V_j(\theta_i^k;\mu_j)||^2\notag\\
     &\leq 2\min_{k<K}||\overline{\nabla V}(\vtheta^k;\vmu)||^2+\frac{2}{N}\sum_{j=1}^{N}\frac{64}{(1-\gamma_j)^6}||\theta_i^k-\theta_j^k||^2.
     \label{eq:thm2_intermediate_Vbound}
\end{align}

From Lemma \ref{lem:bounded_deviation_from_mean}, we have
\begin{align}
    ||\theta_i^k-\theta_j^k|| &=||(\theta_i^k-\bar{\theta}^k)-(\bar{\theta}^k-\theta_j^k)||\notag\\
    &\leq ||\theta_i^k-\bar{\theta}^k||+||\theta_j^k-\bar{\theta}^k||\notag\\
    &\leq \frac{2\alpha D}{1-\sigma_2}.
\end{align}

Plugging this inequality and \eqref{eq:Vbar_bound_thm2} into \eqref{eq:thm2_intermediate_Vbound}, we get
\begin{align*}
&\min_{k<K}||\frac{1}{N}\sum_{j=1}^N\nabla V_j(\theta_i^k;\mu_j)||^2\notag\\
&\quad \leq \frac{4c_1}{K\alpha^2}\sum_{j=1}^N (\frac{1}{1-\gamma_j}+\lambda\text{RE}(\pi_{\theta^0}))+\frac{16\lambda^2}{N}+\frac{2}{N}\sum_{j=1}^N \frac{64}{(1-\gamma_j)^6} \frac{4\alpha^2 D^2}{(1-\sigma_2)^2}\\
&\quad \leq \frac{16}{K\alpha}\sum_{j=1}^N \left(\frac{1}{1-\gamma_j}+\lambda\text{RE}(\pi_{\theta^0})\right)+\frac{16\lambda^2}{N} + \sum_{j=1}^{N}\frac{512D^2\alpha^2}{N(1-\sigma_{2})^2(1-\gamma_j)^6}\cdot
\end{align*}

The proof is completed by recognizing $\rho_i=\mu_i,\,\forall i$.

\subsection{Proof of Theorem \ref{THM:MTPG_TABULAR_SOFTMAX_CONSTANTALPHA}}

When condition \eqref{eq:discounted_visitation_assumption} is observed, we can establish the global optimality condition under the tabular policy.

\begin{prop}
Let $\theta^*=\max_{\theta}V(\theta;\vrho)$. For policy parameter $\theta$, if $||\sum_{i=1}^N \nabla L_i^{\lambda}(\theta;\mu_i)||\leq \frac{\lambda N}{2|\Scal||\Acal|}$, we have 
\begin{align*}
    V(\theta^*;\vrho)-V(\theta;\vrho)\leq 2\lambda N\max_{s\in\Scal,i:s\in\Scal_i}\{\frac{d_{\rho_i}^{\pi_{\theta^{\star}}}(s)}{(1-\gamma_i)\mu_i(s)}\}
\end{align*}
if the environment and the initial state distributions $\vrho$ and $\vmu$ jointly satisfies the discounted visitation match assumption.
\label{THM:SOFTMAX_OPTIMALITY_DETERMINISTIC}
\end{prop}

The proof this proposition is in Section \ref{sec:softmax_optimality_deterministic}. Using the proposition, we can guarantee that $\theta_i^k$ is an $\epsilon$-optimal solution in the objective function by setting $\epsilon=2N\lambda \max_{j,s}\{\frac{d_{\rho_j}^{\pi_{\theta^{\star}}}(s)}{(1-\gamma_j)\mu_j(s)}\}$ and ensuring $||\sum_{j=1}^N \nabla L_j^{\lambda}(\theta_i^k;\mu_j)||\leq \frac{\lambda N}{2|\Scal||\Acal|}$. Solving for $\lambda$ in terms of $\epsilon$, we get
\begin{align}
    \lambda=\frac{\epsilon}{2N\max_{j,s}\{\frac{d_{\rho_j}^{\pi_{\theta^{\star}}}(s)}{(1-\gamma_j)\mu_j(s)}\}}.
\end{align}

Now we bound the norm of the gradient.
\begin{align}
    \min_{k<K}||\sum_{j=1}^N \nabla L_j^{\lambda}(\theta_i^k;\mu_j)||&=\min_{k<K}||\sum_{j=1}^N \nabla L_j^{\lambda}(\theta_j^k;\mu_j)+\sum_{j=1}^N \left(\nabla L_j^{\lambda}(\theta_i^k;\mu_j)-\nabla L_j^{\lambda}(\theta_j^k)\right)||\notag\\
    &\leq \min_{k<K}||N\overline{\nabla \mL}^{\lambda}(\vtheta^k;\vmu)||+\sum_{j=1}^N||\nabla L_j^{\lambda}(\theta_i^k;\mu_j)-\nabla L_j^{\lambda}(\theta_j^k;\mu_j)||\notag\\
    &\leq N\min_{k<K}||\overline{\nabla \mL}^{\lambda}(\vtheta^k;\vmu)||+\sum_{j=1}^N \beta^{\lambda}_i||\theta_i^k-\theta_j^k||,
    \label{eq:softmax_bound_Llambda_inproof}
\end{align}
where the last inequality uses the smoothness property of $L_i^{\lambda}$.Combining Lemma \ref{lem:V_softmax_Hessian_Lipschitz} and Lemma \ref{lem:RE_softmax_Hessian_Lipschitz}, $\beta_i^{\lambda}=\frac{8}{(1-\gamma_i)^3}+\frac{2\lambda}{|\Scal|}$. We have a bound on the first term in \eqref{eq:graident_convergence_L}, and now we bound the second term using Lemma \ref{lem:bounded_deviation_from_mean}.
\begin{align}
    ||\theta_i^k-\theta_j^k|| &=||(\theta_i^k-\bar{\theta}^k)-(\bar{\theta}^k-\theta_j^k)||\notag\\
    &\leq ||\theta_i^k-\bar{\theta}^k||+||\theta_j^k-\bar{\theta}^k||\notag\\
    &\leq \frac{2\alpha D}{1-\sigma_2}
\end{align}

Plug this into \eqref{eq:softmax_bound_Llambda_inproof}, 
\begin{align}
    \min_{k<K}||\sum_{j=1}^N \nabla L_j^{\lambda}(\theta_i^k;\mu_j)||&\leq N\min_{k<K}||\overline{\nabla \mL}^{\lambda}(\vtheta^k;\vmu)||+\sum_{j=1}^N \beta^{\lambda}_i||\theta_i^k-\theta_j^k||\notag\\
    &\leq N\sqrt{\frac{c_1}{K\alpha^2}\sum_{j=1}^N (\frac{1}{1-\gamma_j}+\lambda\text{RE}(\pi_{\theta^0}))}+\sum_{j=1}^N \beta^{\lambda}_i\frac{2\alpha D}{1-\sigma_2}\\
    &\leq N\sqrt{\frac{c_1}{K\alpha^2}\sum_{j=1}^N (\frac{1}{1-\gamma_j}+\lambda\text{RE}(\pi_{\theta^0}))}+\frac{2\alpha \beta^{\lambda} D}{1-\sigma_2}
\end{align}

To ensure $\min_{k<K}||\sum_{j=1}^N \nabla L_j^{\lambda}(\theta_i^k;\mu_j)||\leq \frac{\lambda N}{2|\Scal||\Acal|}$, we make
\begin{align}
    N\sqrt{\frac{c_1}{K\alpha^2}\sum_{j=1}^N (\frac{1}{1-\gamma_j}+\lambda\text{RE}(\pi_{\theta^0}))}+\frac{2\alpha \beta^{\lambda} D}{1-\sigma_2}\leq \frac{\lambda N}{2|\Scal||\Acal|}
\end{align}

Solving for K, we get
\begin{align}
    K &\geq \frac{c_1 N^2\left(\sum_{j=1}^N(\frac{1}{1-\gamma_j}+\lambda\text{RE}(\pi_{\theta^0}))\right)}{\alpha^2\left(\frac{\lambda N}{2|\Scal||\Acal|}-\frac{2\alpha \beta^{\lambda} D}{1-\sigma_2}\right)^2}\\
    &= \frac{c_1 N^2\left(\sum_{j=1}^N(\frac{1}{1-\gamma_j}+\lambda\text{RE}(\pi_{\theta^0}))\right)}{\alpha^2\left(\frac{\epsilon c_2}{4|\Scal||\Acal|}-\frac{2\alpha D}{1-\sigma_2}\sum_{j=1}^N\left( \frac{8}{(1-\gamma_j)^3}+\frac{\epsilon c_2}{N|\Scal|}\right)\right)^2},
\end{align}
where we used the fact that $\frac{\lambda N}{2|\Scal||\Acal|}-\frac{2\alpha \beta^{\lambda} D}{1-\sigma_2}>0$, if $\alpha<\frac{\lambda N(1-\sigma_2)}{4\beta^{\lambda}D|\Scal||\Acal|}$.

\subsection{Proof of Proposition \ref{THM:SOFTMAX_OPTIMALITY_DETERMINISTIC}}
\label{sec:softmax_optimality_deterministic}
From the assumption \eqref{eq:discounted_visitation_assumption}, we define
\begin{align}
\frac{d_{i,\rho_i}^{\pi_{\theta^*}}(s)}{d_{i,\mu_i}^{\pi_{\theta}}(s)}=\frac{d_{j,\rho_j}^{\pi_{\theta^*}}(s)}{d_{j,\mu_j}^{\pi_{\theta}}(s)}\triangleq \tilde{d}(s),\qquad \forall s:s\in\Scal_i\cap\Scal_j,\; \forall i,j.   
\end{align}

\citet{kakade2002approximately} introduced the performance difference lemma that relates the value function of two policies. We use this lemma in our analysis.\\

\begin{lemma}
For any policy $\pi$ and $\tilde{\pi}$ operating in environment $i$ under the initial state distribution $\rho_i$,
\begin{align}
    V_i^{\pi}\left(\rho_i\right)-V_i^{\tilde{\pi}}\left(\rho_i\right)=\frac{1}{1-\gamma_i} \mathbb{E}_{s \sim d_{i,\rho_i}^{\pi}} \mathbb{E}_{a \sim \pi(\cdot | s)}\left[A^{\pi^{\prime}}(s, a)\right].
\end{align}
\label{lem:performance_diff}
\end{lemma}

By Lemma \ref{lem:performance_diff},
\begin{align}
V(\theta^*;\vrho)-V(\theta;\vrho)&=\sum_{i=1}^{N}\frac{1}{1-\gamma_i} \sum_{s\in\Scal_{i}}\sum_{a\in\Acal} d_{i,\rho_i}^{\pi_{\theta^{\star}}}(s) \pi_{\theta^{\star}}(a \,|\, s) A_i^{\pi_{\theta}}(s, a)\notag\\
&= \sum_{s\in\Scal}\sum_{a\in\Acal}\pi_{\theta^{\star}}(a \,|\, s)\sum_{i:s\in\Scal_i}\frac{1}{1-\gamma_i} d_{i,\rho_i}^{\pi_{\theta^{\star}}}(s)  A_i^{\pi_{\theta}}(s, a)\notag\\
&\leq \sum_{s\in\Scal}\max_{a\in\Acal}\sum_{i:s\in\Scal_i}\frac{1}{1-\gamma_i} d_{i,\rho_i}^{\pi_{\theta^{\star}}}(s)  A_i^{\pi_{\theta}}(s, a)\notag\\
&= \sum_{s\in\Scal}\max_{a\in\Acal}\sum_{i:s\in\Scal_i}\frac{d_{i,\rho_i}^{\pi_{\theta^{\star}}}(s) }{d_{i,\mu_i}^{\pi_{\theta}}(s) }\frac{d_{i,\mu_i}^{\pi_{\theta}}(s)}{1-\gamma_i}  A_i^{\pi_{\theta}}(s, a)\notag\\
&= \sum_{s\in\Scal}\tilde{d}(s)\max_{a\in\Acal}\sum_{i:s\in\Scal_i}\frac{d_{i,\mu_i}^{\pi_{\theta}}(s)}{1-\gamma_i}  A_i^{\pi_{\theta}}(s, a)\notag\\
&\leq \max_{s\in\Scal,i:s\in\Scal_i}\{\frac{d_{i,\rho_i}^{\pi_{\theta^*}}(s)}{d_{i,\mu_i}^{\pi_{\theta}}(s)}\}\sum_{s\in\Scal}\max_{a\in\Acal}\sum_{i:s\in\Scal_i}\frac{d_{i,\mu_i}^{\pi_{\theta}}(s)}{1-\gamma_i}  A_i^{\pi_{\theta}}(s, a)\notag\\
&\leq \max_{s\in\Scal,i:s\in\Scal_i}\{\frac{d_{i,\rho_i}^{\pi_{\theta^*}}(s)}{d_{i,\mu_i}^{\pi_{\theta}}(s)}\} |\Scal|\frac{2\lambda N}{|\Scal|}\notag\\
&=2\lambda N\max_{s\in\Scal,i:s\in\Scal_i}\{\frac{d_{i,\rho_i}^{\pi_{\theta^*}}(s)}{d_{i,\mu_i}^{\pi_{\theta}}(s)}\}\notag\\
&=2\lambda N\max_{s\in\Scal,i:s\in\Scal_i}\{\frac{d_{i,\rho_i}^{\pi_{\theta^*}}(s)}{(1-\gamma_i)\mu_i(s)}\}
\end{align}

The sixth line follows since $\max_{a\in\Acal}\sum_{i:s\in\Scal_i}\frac{d_{i,\mu_i}^{\pi_{\theta}}(s)}{1-\gamma_i}  A_i^{\pi_{\theta}}(s, a) \geq 0,\,\forall s$. The last inequality uses the fact that $d_{i,\mu_i}^{\pi}(s)\geq(1-\gamma_i)\mu_i(s),\,\text{element-wise,}\,\forall \pi$, which simply follows from the definition of $d_{i,\mu_i}^{\pi}(s)$. The seventh line uses
\begin{align}
    \max_{a\in\Acal}\sum_{i:s\in\Scal_i} \frac{d_{i,\mu_i}^{\pi_{\theta}}(s)}{1-\gamma_i}A_i^{\pi_{\theta}}(s, a)\leq\frac{2\lambda N}{|\Scal|},
\end{align}
which we now prove. To show this, we only have to prove this is true for those $(s,a)$ where $\sum_{i:s\in\Scal_i} \frac{d_{i,\mu_i}^{\pi_{\theta}}(s)}{1-\gamma_i}A_i^{\pi_{\theta}}(s, a)\geq 0$. The gradient of $\theta$ under the softmax parameterization in environment $i$ is
\begin{align}
    \frac{\partial L_i^{\lambda}(\theta;\mu_i)}{\partial \theta_{s, a}}=\frac{1}{1-\gamma_i} d_{i,\mu_i}^{\pi_{\theta}}(s) \pi_{\theta}(a \,|\, s) A_i^{\pi_{\theta}}(s, a)+\frac{\lambda}{|\mathcal{S}|}\left(\frac{1}{|\mathcal{A}|}-\pi_{\theta}(a | s)\right).
    \label{eq:softmax_grad_inproof}
\end{align}
From our assumption $||\sum_{i=1}^N \nabla L_i^{\lambda}(\theta;\mu_i)||\leq \frac{\lambda N }{2|\Scal||\Acal|}$, we know that for all $(s,a)$ such that $\sum_{i:s\in\Scal_i} \frac{d_{i,\mu_i}^{\pi_{\theta}}(s)}{1-\gamma_i}  A_i^{\pi_{\theta}}(s, a)\geq 0$,
\begin{align} 
\frac{\lambda N}{2|\Scal||\Acal|} &\geq \sum_{i=1}^{N}\frac{\partial L_i^{\lambda}(\theta;\mu_i)}{\partial \theta_{s, a}}\notag\\
&=\sum_{i:s\in\Scal_i}\frac{1}{1-\gamma_i} d_{i,\mu_i}^{\pi_{\theta}}(s) \pi_{\theta}(a | s) A_i^{\pi_{\theta}}(s, a)+\sum_{i=1}^{N}\frac{\lambda}{|\Scal|}\left(\frac{1}{|\Acal|}-\pi_{\theta}(a | s)\right)\notag \\ 
&\geq 0+\sum_{i=1}^{N}\frac{\lambda}{|\Scal|}\left(\frac{1}{|\Acal|}-\pi_{\theta}(a | s)\right)\notag\\
&\geq \frac{\lambda N}{|\Scal|}\left(\frac{1}{|\Acal|}-\pi_{\theta}(a | s)\right).
\end{align}

Rearranging the terms,
\begin{align}
    \pi_{\theta}(a \,|\, s) \geq \frac{1}{|\Acal|}-\frac{|\Scal|}{\lambda N}\frac{\lambda N}{2|\Scal||\Acal|} \geq \frac{1}{2|\Acal|}.
    \label{eq:softmax_pi_lowerbound}
\end{align}

Re-writing Eq. \eqref{eq:softmax_grad_inproof} and summing over environments,
\begin{align}
\sum_{i=1}^{N} \frac{d_{i,\mu_i}^{\pi_{\theta}}(s)}{1-\gamma_i} A_i^{\pi_{\theta}}(s, a) &=\sum_{i:s\in\Scal_i}\frac{1}{\pi_{\theta}(a | s)} \frac{\partial L_i^{\lambda}(\theta;\mu_i)}{\partial \theta_{s, a}}-\sum_{i=1}^{N}\frac{\lambda}{|\Scal|}\left(\frac{1}{\pi_{\theta}(a \,|\, s)|\Acal|}-1\right) \notag\\ 
&\leq \frac{1}{\pi_{\theta}(a | s)}\sum_{i:s\in\Scal_i}\frac{\partial L_i^{\lambda}(\theta;\mu_i)}{\partial \theta_{s, a}}+\sum_{i=1}^{N}\frac{\lambda}{|\Scal|} \notag\\ 
&\leq 2|\Acal|\frac{\lambda N}{2|\Scal||\Acal|}+\frac{\lambda N}{|\Scal|}\notag\\
&\leq\frac{2\lambda N}{|\Scal|},
\end{align}
where the second last line uses inequality \eqref{eq:softmax_pi_lowerbound}.

\subsection{Derivation of the gradient
\eqref{SEC:PROB:OBJ_THETA_REG:GRAD}}
\label{sec:policy_gradient_MA_derivation}
Here we just derive the gradient for $V_i^{\pi_{\theta}}$. The gradient of $L_i^{\lambda}$ can be easily computed from the gradient of $\nabla V_i^{\pi_{\theta}}$ by adding the gradient of the entropy regularizer.

By definition, 
\begin{align*}
V_i^{\pi_{\theta}}(s_{i}) &= \Eset\left[\sum_{k=0}^{\infty}\gamma_{i}^{k}\Rcal_{i}(s^{k}_{i},a_{i}^{k})\,|\,s_{i}^{0} = s_{i}\right],\qquad a_{i}^{k} \sim \pi_{\theta}(s_{i}^{k})\\
&= \sum_{a_i\in\Acal}\pi_{\theta}(a_{i}\,|\,s_{i})Q_i^{\pi_{\theta}}(s_{i},a_{i})\\
&=\sum_{a_i\in\Acal}\pi_{\theta}(a_{i}\,|\,s_{i})\Eset_{s_i'\in\Scal_i}\left[\Rcal(s_i,a_i) + \gamma_i V_i^{\pi_{\theta}}(s_i')\right],
\end{align*}

which implies

\begin{align*}
\frac{\partial V_i^{\pi_{\theta}}(s_i)}{\partial \theta} &= \sum_{a_i\in\Acal}\left[Q_i^{\pi_{\theta}}(s_i,a_i)\frac{\partial \pi_{\theta}(a_{i}\,|\,s_{i})}{\partial \theta}+\pi_{\theta}(a_{i}\,|\,s_{i})\frac{\partial Q_i^{\pi_{\theta}}(s_i,a_i)}{\partial \theta}\right]\\
&= \sum_{a_i\in\Acal}\pi_{\theta}(a_{i}\,|\,s_{i})Q_i^{\pi_{\theta}}(s_i,a_i)\frac{\nabla_{\theta} \pi_{\theta}(a_{i}\,|\,s_{i})}{\pi_{\theta}(a_{i}\,|\,s_{i})}\\
&\hspace{50pt} +\sum_{a_i\in\Acal}\pi_{\theta}(a_{i}\,|\,s_{i})\frac{\partial}{\partial \theta}\Eset_{s_i'\in\Scal_i}\left[\Rcal(s_i,a_i) + \gamma_i V_i^{\pi_{\theta}}(s_i')\right]\\
&= \sum_{a_i\in\Acal}\pi_{\theta}(a_{i}\,|\,s_{i})Q_i^{\pi_{\theta}}(s_i,a_i)\nabla_{\theta}\ln \pi_{\theta}(a_{i}\,|\,s_{i})\\
&\hspace{50pt} +\gamma_i\sum_{a_i\in\Acal}\pi_{\theta}(a_{i}\,|\,s_{i})\sum_{s_i'\in\Scal_i}p_i(s_i'\,|\,s_i,a_i)\frac{\partial}{\partial \theta} V_i^{\pi_{\theta}}(s_i')\\
&= \sum_{a_i\in\Acal}\pi_{\theta}(a_i\,|\,s_i)Q_i^{\pi_{\theta}}(s_i,a_i)\nabla_{\theta}\ln \pi_{\theta}(a_i\,|\,s_i)  \\ 
&\hspace{50pt} +\gamma_i\sum_{a_i\in\Acal}\pi_{\theta}(a_i\,|\,s_i)\sum_{s_i'\in\Scal_i}p_i(s_i'\,|\,s_i,a_i) \\
&\hspace{100pt}\times\sum_{a_i'\in\Acal}\pi_{\theta}(a_i'\,|\,s_i')Q_i^{\pi_{\theta}}(s_i',a_i')\nabla_{\theta}\ln \pi_{\theta}(a_i'\,|\,s_i')\\
&\hspace{50pt} + \gamma_i^2\sum_{a_i'\in\Acal}\pi_{\theta}(a_i\,|\,s_i)\sum_{s_i'\in\Scal_i}p_i(s_i'\,|\,s_i,a_i)\\
&\hspace{100pt}\times\sum_{a'\in\Acal}\pi_{\theta}(a_i'\,|\,s_i')\sum_{s_i''\in\Scal_i}p_i(s_i''\,|\,s_i',a_i')\frac{\partial}{\partial \theta}V_i^{\pi_{\theta}}(s_i'')\\
&= \sum_{k=0}^{\infty}\sum_{s_i'\in\Scal_i}\gamma_i^{k}P_{\pi_{\theta}}(s_i^{k} = s_i'\,|\,s_i)\sum_{a_i'\in\Acal}\pi_{\theta}(a_i'\,|\,s_i')Q_i^{\pi_{\theta}}(s_i',a_i')\nabla_{\theta}\ln \pi_{\theta}(a_i'\,|\,s_i')\\
&= \frac{1}{1-\gamma_i}\Eset_{s_i'\sim d^{\pi_{\theta}}_{s_i}(\cdot)}\Eset_{a_i'\sim\pi_{\theta}(\cdot\,|\,s_i')}Q_i^{\pi_{\theta}}(s_i',a_i')\nabla_{\theta}\ln \pi_{\theta}(a_i'\,|\,s_i')\\
&= \frac{1}{1-\gamma_i}\Eset_{s_i'\sim d^{\pi_{\theta}}_{s_i}(\cdot)}\Eset_{a_i'\sim\pi_{\theta}(\cdot\,|\,s_i')}\left(Q_i^{\pi_{\theta}}(s_i',a_i')-V_i^{\pi_{\theta}}(s_i')+V_i^{\pi_{\theta}}(s_i')\right)\nabla_{\theta}\ln \pi_{\theta}(a_i'\,|\,s_i')\\
&= \frac{1}{1-\gamma_i}\Eset_{s_i'\sim d^{\pi_{\theta}}_{s_i}(\cdot)}\Eset_{a_i'\sim\pi_{\theta}(\cdot\,|\,s_i')}A_i^{\pi_{\theta}}(s_i',a_i')\nabla_{\theta}\ln \pi_{\theta}(a_i'\,|\,s_i'),
\end{align*}

where the last equation follows since
\begin{align*}
    \sum_{a'}\pi_{\theta}(a'|s')V_i^{\pi_{\theta}}(s_i')\nabla_{\theta}\ln \pi_{\theta}(a_i'\,|\,s_i') = \sum_{a'} V_i^{\pi_{\theta}}(s_i')\nabla_{\theta}\pi_{\theta}(a'|s')=V_i^{\pi_{\theta}}(s_i')\nabla_{\theta}\sum_{a'}\pi_{\theta}(a'|s')=0.
\end{align*} 

Let $\mathbb{1}[\cdot]$ denote the indicator function of argument condition. We observe that under the softmax parameterization
\begin{align*}
    \frac{\partial \ln \pi_{\theta}(a' | s')}{\partial \theta_{s,a}}&= \frac{\partial}{\partial \theta_{s,a}}\left(\theta_{s',a'}-\ln \sum_{a''}\text{exp}(\theta_{s',a''})\right)\\
    &=\mathbb{1}(s'=s,a'=a)-\mathbb{1}(s'=s)\frac{\text{exp}(\theta_{s',a})}{\sum_{a''}\text{exp}(\theta_{s',a''})}\\
    &=\mathbb{1}\left[s'=s\right]\left(\mathbb{1}\left[a'=a\right]-\pi_{\theta}\left(a | s'\right)\right).
\end{align*}

Therefore,
\begin{align*}
\frac{\partial}{\partial \theta_{s,a}} V_i^{\pi_{\theta}}(\rho_i) &=\frac{\partial}{\partial\theta_{s,a}}\Eset_{s_{i}\sim\rho_{i}}\left[V_i^{\pi_{\theta}}(s_{i})\right]\\
&=\Eset_{s_{i}\sim\rho_{i}}\left[\frac{\partial V_i^{\pi_{\theta}}(s_i)}{\partial\theta_{s,a}}\right]\\
&=\frac{1}{1-\gamma_i}\Eset_{s_{i}\sim\rho_{i}}\Eset_{s_i'\sim d^{\pi_{\theta}}_{s_i}(\cdot)}\Eset_{a_i'\sim\pi_{\theta}(\cdot\,|\,s_i')}A_i^{\pi_{\theta}}(s_i',a_i')\frac{\partial \ln \pi_{\theta}(a_i'\,|\,s_i')}{\partial \theta_{s,a}}\\
&=\frac{1}{1-\gamma_i}\Eset_{s_{i}\sim\rho_{i}}\Eset_{s_i'\sim d^{\pi_{\theta}}_{s_i}(\cdot)}\Eset_{a_i'\sim\pi_{\theta}(\cdot\,|\,s_i')}A_i^{\pi_{\theta}}(s_i',a_i')\mathbb{1}\left[s_i'=s\right]\left(\mathbb{1}\left[a_i'=a\right]-\pi_{\theta}\left(a | s_i'\right)\right)\\
&=\frac{1}{1-\gamma_i}\Eset_{s_{i}\sim\rho_{i}}d^{\pi_{\theta}}_{s_i}(s)\pi_{\theta}(a|s)A_i^{\pi_{\theta}}(s,a)\\
&\hspace{50pt}-\frac{\pi_{\theta}\left(a | s\right)}{1-\gamma_i}\Eset_{s_{i}\sim\rho_{i}}\Eset_{s_i'\sim d^{\pi_{\theta}}_{s_i}(\cdot)}\mathbb{1}\left[s_i'=s\right]\sum_{a_i'}\pi_{\theta}(a_i'|s_i')A_i^{\pi_{\theta}}(s_i',a_i')\\
&=\frac{1}{1-\gamma_i}d^{\pi_{\theta}}_{\rho_i}(s)\pi_{\theta}(a|s)A_i^{\pi_{\theta}}(s,a)
\end{align*}

\subsection{Lipschitz, smoothness, and Hessian Lipschitz constants}
\begin{lemma}
Let $\pi_{\alpha} \triangleq \pi_{\vtheta+\alpha \vu}$, where $\vu$ is a unit vector and $\tilde{V}_i(\alpha)\triangleq V_i^{\pi_{\alpha}}(s_i)$. If 
\begin{align}
    \sum_{a \in \mathcal{A}}\left|\left.\frac{d \pi_{\alpha}\left(a | s_{0}\right)}{d \alpha}\right|_{\alpha=0}\right|\leq C', \quad \sum_{a \in \mathcal{A}}\left|\left.\frac{d^{2} \pi_{\alpha}\left(a | s_{0}\right)}{d \alpha^{2}}\right|_{\alpha=0} \right| \leq C'', \quad \text{and}
    \sum_{a \in \mathcal{A}}\left|\left.\frac{d^3 \pi_{\alpha}\left(a | s_{0}\right)}{d \alpha^3}\right|_{\alpha=0}\right|\leq C''' ,
\end{align}
then we have
\begin{align}
    &\max _{||\vu||_{2}=1}\left|\left.\frac{d \tilde{V}_i(\alpha)}{d \alpha}\right|_{\alpha=0} \right| \leq \frac{C'}{(1-\gamma_i)^2},\notag\\
    &\max _{||\vu||_{2}=1}\left|\left.\frac{d^{2} \tilde{V}_i(\alpha)}{d \alpha^{2}}\right|_{\alpha=0} \right| \leq \frac{C''}{(1-\gamma_i)^2}+\frac{2\gamma_i C'^2}{(1-\gamma_i)^{3}},\,\text{and}\notag\\
    &\max _{||\vu||_{2}=1}\left|\left.\frac{d^{3} \tilde{V}_i(\alpha)}{d \alpha^{3}}\right|_{\alpha=0} \right| \leq \frac{C'''}{(1-\gamma_i)^2}+\frac{6\gamma_i C'C''}{(1-\gamma_i)^{3}}+\frac{6\gamma_i^2 C'^3}{(1-\gamma_i)^{4}}
\end{align}
\label{lem:bounded_thirdderivative}
\end{lemma}
\begin{proof}
The proof uses a similar technique to Lemma E.2 of \citet{AgarwalKLM2019}, which proves the second derivative is bounded. Here we also show the first and the third derivative is bounded. We use $\tilde{P}_i(\alpha)$ to denote the state-action transition matrix in environment $i$.
\begin{align}
    [\tilde{P}_i(\alpha)]_{(s, a) \rightarrow\left(s^{\prime}, a^{\prime}\right)}=\pi_{\alpha}\left(a^{\prime} | s^{\prime}\right) P_i\left(s^{\prime} | s, a\right)
\end{align}
Differentiating with respect to $\alpha$, we get 
\begin{align}
    \left[\left.\frac{d \tilde{P}_i(\alpha)}{d \alpha}\right|_{\alpha=0}\right]_{(s, a) \rightarrow\left(s^{\prime}, a^{\prime}\right)}=\left.\frac{d \pi_{\alpha}\left(a^{\prime} | s^{\prime}\right)}{d \alpha}\right|_{\alpha=0} P_i\left(s^{\prime} | s, a\right),
\end{align}
which implies that for any $\vx$,
\begin{align}
    \left[\left.\frac{d \tilde{P}_i(\alpha)}{d \alpha}\right|_{\alpha=0} x\right]_{s, a}=\left.\sum_{a^{\prime}, s^{\prime}} \frac{d \pi_{\alpha}\left(a^{\prime} | s^{\prime}\right)}{d \alpha}\right|_{\alpha=0} P_i\left(s^{\prime} | s, a\right) x_{a^{\prime}, s^{\prime}}
\end{align}

We can bound the $\ell_{\infty}$ norm of this as
\begin{align}
    \max _{||\vu||_{2}=1}\left\|\frac{d \tilde{P}_i(\alpha)}{d \alpha} \vx\right\|_{\infty}&=\max_{s,a}\max _{||\vu||_{2}=1}\left|\left[\left.\frac{d \tilde{P}_i(\alpha)}{d \alpha}\right|_{\alpha=0} \vx\right]_{s, a}\right| \notag\\
    &=\max_{s,a}\max _{||\vu||_{2}=1}\left|\left.\sum_{a', s'} \frac{d \pi_{\alpha}\left(a' | s'\right)}{d \alpha}\right|_{\alpha=0} P_i\left(s' | s, a\right) \vx_{a', s'} \right| \notag\\ 
    &\leq \max_{s,a} \sum_{a', s'}\left|\left.\frac{d \pi_{\alpha}\left(a' | s'\right)}{d \alpha}\right|_{\alpha=0}\right|P_i\left(s' | s, a\right) | \vx_{a', s'} | \notag\\ 
    &\leq \max_{s,a} \sum_{s'} P_i\left(s' | s, a\right)||\vx||_{\infty} \sum_{a'}\left|\left.\frac{d \pi_{\alpha}\left(a' | s'\right)}{d \alpha}\right|_{\alpha=0} \right| \notag\\ 
    & \leq C'||\vx||_{\infty}
    \label{eq:Bounded_Transition_Derivative1}
\end{align}

Using the same approach, we can bound
\begin{align}
    \max _{||\vu||_{2}=1}\left\|\frac{d^2 \tilde{P}_i(\alpha)}{d \alpha^2} \vx\right\|_{\infty} \leq C''||\vx||_{\infty}, \, \text{and} \max _{||\vu||_{2}=1}\left\|\frac{d^3 \tilde{P}_i(\alpha)}{d \alpha^3} \vx\right\|_{\infty} \leq C'''||\vx||_{\infty}.
    \label{eq:Bounded_Transition_Derivative23}
\end{align}

With $M(\alpha):=(\mI-\gamma_i \tilde{P}_i(\alpha))^{-1}$, we re-writing the Bellman equation in the matrix form, 
\begin{align}
    Q^{\alpha}(s_{0}, a_{0})=e_{\left(s_{0}, a_{0}\right)}^{T}(\mI-\gamma_i \tilde{P}_i(\alpha))^{-1} r=e_{\left(s_{0}, a_{0}\right)}^{T} M(\alpha) r.
\end{align}

Taking the first, second, and third derivative of $Q^{\alpha}(s_{0}, a_{0})$ with respect to $\alpha$,
\begin{align}
    \frac{d Q^{\alpha}\left(s_{0}, a\right)}{d \alpha}=\gamma_i e_{\left(s_{0}, a\right)}^{T} M(\alpha) \frac{d \tilde{P}_i(\alpha)}{d \alpha} M(\alpha) r,
\end{align}
\begin{align}
    \frac{d^{2} Q^{\alpha}\left(s_{0}, a_{0}\right)}{(d \alpha)^{2}}=2 \gamma_i^{2} e_{\left(s_{0}, a_{0}\right)}^{T} M(\alpha) \frac{d \tilde{P}_i(\alpha)}{d \alpha} M(\alpha) \frac{d \tilde{P}_i(\alpha)}{d \alpha} M(\alpha) r &\notag \\+\gamma_i e_{\left(s_{0}, a_{0}\right)}^{T} M(\alpha) \frac{d^{2} \tilde{P}_i(\alpha)}{d \alpha^{2}} M(\alpha) r, &
\end{align}
\begin{align}
    \frac{d^{3} Q^{\alpha}\left(s_{0}, a_{0}\right)}{(d \alpha)^{3}}=6 \gamma_i^{3} e_{\left(s_{0}, a_{0}\right)}^{T} M(\alpha) \frac{d \tilde{P}_i(\alpha)}{d \alpha} M(\alpha) \frac{d \tilde{P}_i(\alpha)}{d \alpha} M(\alpha) \frac{d \tilde{P}_i(\alpha)}{d \alpha} M(\alpha) r &\notag
    \\+3 \gamma_i^{2} e_{\left(s_{0}, a_{0}\right)}^{T} M(\alpha) \frac{d^2 \tilde{P}_i(\alpha)}{d \alpha^2} M(\alpha) \frac{d \tilde{P}_i(\alpha)}{d \alpha} M(\alpha) r &\notag
    \\+3 \gamma_i^{2} e_{\left(s_{0}, a_{0}\right)}^{T} M(\alpha) \frac{d \tilde{P}_i(\alpha)}{d \alpha} M(\alpha) \frac{d^2 \tilde{P}_i(\alpha)}{d \alpha^2} M(\alpha) r &\notag
    \\+\gamma_i e_{\left(s_{0}, a_{0}\right)}^{T} M(\alpha) \frac{d^{3} \tilde{P}_i(\alpha)}{d \alpha^{3}} M(\alpha) r &
\end{align}

Using $M(\alpha)\vct{1}=(\mathrm{I}-\gamma_i \tilde{P}_i(\alpha))^{-1}\vct{1}=\sum_{n=0}^{\infty} \gamma_i^{n} \tilde{P}_i(\alpha)^{n}\vct{1}=\frac{1}{1-\gamma}\vct{1}$ and inequalities \eqref{eq:Bounded_Transition_Derivative1} and \eqref{eq:Bounded_Transition_Derivative23}, we have
\begin{align} 
\max _{||\vu||_{2}=1}\left|\left.\frac{d Q^{\alpha}\left(s_{0}, a\right)}{d \alpha}\right|_{\alpha=0} \right| & \leq\left\|\gamma_i M(\alpha) \frac{d \tilde{P}_i(\alpha)}{d \alpha} M(\alpha) r\right\|_{\infty} \notag\\ 
& \leq \frac{\gamma_i C'}{(1-\gamma_i)^{2}},
\end{align}
\begin{align} 
\max _{||\vu||_{2}=1}\left|\frac{d^{2} Q^{\alpha}\left(s_{0}, a_{0}\right)}{d \alpha^{2}}\right|_{\alpha=0} | &\leq 2 \gamma_i^{2}\left\|M(\alpha) \frac{d \tilde{P}_i(\alpha)}{d \alpha} M(\alpha) \frac{d \tilde{P}_i(\alpha)}{d \alpha} M(\alpha) r\right\|_{\infty} \notag\\ &\hspace{100pt}+ \gamma_i\left\|M(\alpha) \frac{d^{2} \tilde{P}_i(\alpha)}{d \alpha^{2}} M(\alpha) r\right\|_{\infty} \\ 
&\leq \frac{2 \gamma_i^{2} C'^{2}}{(1-\gamma_i)^{3}}+\frac{\gamma_i C''}{(1-\gamma_i)^{2}}
\end{align}
\begin{align} 
\max _{||\vu||_{2}=1}\left|\frac{d^{3} Q^{\alpha}\left(s_{0}, a_{0}\right)}{d \alpha^{3}}\right|_{\alpha=0} | &\leq 6 \gamma_i^{3}\left\|M(\alpha) \frac{d \tilde{P}_i(\alpha)}{d \alpha} M(\alpha) \frac{d \tilde{P}_i(\alpha)}{d \alpha} M(\alpha) \frac{d \tilde{P}_i(\alpha)}{d \alpha} M(\alpha) r\right\|_{\infty} \notag\\
&\hspace{50pt}+3\gamma_i^{2}\left\|M(\alpha) \frac{d \tilde{P}_i(\alpha)}{d \alpha} M(\alpha) \frac{d^2 \tilde{P}_i(\alpha)}{d \alpha^2} M(\alpha) r\right\|_{\infty} \notag\\ 
&\hspace{50pt}+3\gamma_i^{2}\left\|M(\alpha) \frac{d^2 \tilde{P}_i(\alpha)}{d \alpha^2} M(\alpha) \frac{d \tilde{P}_i(\alpha)}{d \alpha} M(\alpha) r\right\|_{\infty} \notag\\ 
&\hspace{50pt}+ \gamma_i\left\|M(\alpha) \frac{d^{3} \tilde{P}_i(\alpha)}{d \alpha^{3}} M(\alpha) r\right\|_{\infty} \notag\\ 
&\leq \frac{6 \gamma_i^{3} C'^{3}}{(1-\gamma_i)^{4}}+\frac{3\gamma_i^2 C'C''}{(1-\gamma_i)^{3}}+\frac{3\gamma_i^2 C'C''}{(1-\gamma_i)^{3}}+\frac{\gamma_i C'''}{(1-\gamma_i)^{2}}\notag\\
&=\frac{6 \gamma_i^{3} C'^{3}}{(1-\gamma_i)^{4}}+\frac{6\gamma_i^2 C'C''}{(1-\gamma_i)^{3}}+\frac{\gamma_i C'''}{(1-\gamma_i)^{2}}
\end{align}

By the definition of $\tilde{V}_i(\alpha)$,
\begin{align}
    \tilde{V}_i(\alpha)=\sum_{a} \pi_{\alpha}\left(a | s_{0}\right) Q^{\alpha}\left(s_{0}, a\right).
\end{align}

Taking the first derivative of $\tilde{V}_i(\alpha)$ with respect to $\alpha$, 
\begin{align}
    \frac{d \tilde{V}_i(\alpha)}{d \alpha}=\sum_{a} \frac{d \pi_{\alpha}\left(a | s_{0}\right)}{d \alpha} Q_i^{\alpha}\left(s_{0}, a\right)+\sum_{a} \pi_{\alpha}\left(a | s_{0}\right) \frac{d Q_i^{\alpha}\left(s_{0}, a\right)}{d \alpha}.
\end{align}

Taking the second derivative of $\tilde{V}_i(\alpha)$ with respect to $\alpha$, 
\begin{align}
    \frac{d^{2} \tilde{V}_i(\alpha)}{d \alpha^{2}}&=\sum_{a} \frac{d^{2} \pi_{\alpha}\left(a | s_{0}\right)}{d \alpha^{2}} Q_i^{\alpha}\left(s_{0}, a\right)+2 \sum_{a} \frac{d \pi_{\alpha}\left(a | s_{0}\right)}{d \alpha} \frac{d Q_i^{\alpha}\left(s_{0}, a\right)}{d \alpha}\notag\\
    &\hspace{50pt}+\sum_{a} \pi_{\alpha}\left(a | s_{0}\right) \frac{d^{2} Q_i^{\alpha}\left(s_{0}, a\right)}{d \alpha^{2}}.
\end{align}

Taking the third derivative of $\tilde{V}_i(\alpha)$ with respect to $\alpha$, 
\begin{align}
    \frac{d^{3} \tilde{V}_i(\alpha)}{d \alpha^{3}}&=\sum_{a} \frac{d^{3} \pi_{\alpha}\left(a | s_{0}\right)}{d \alpha^{3}} Q^{\alpha}\left(s_{0}, a\right)+3 \sum_{a} \frac{d^2 \pi_{\alpha}\left(a | s_{0}\right)}{d \alpha^2} \frac{d Q^{\alpha}\left(s_{0}, a\right)}{d \alpha}\notag\\
    &\hspace{50pt}+3 \sum_{a} \frac{d \pi_{\alpha}\left(a | s_{0}\right)}{d \alpha} \frac{d^2 Q^{\alpha}\left(s_{0}, a\right)}{d \alpha^2}+\sum_{a} \pi_{\alpha}\left(a | s_{0}\right) \frac{d^{3} Q^{\alpha}\left(s_{0}, a\right)}{d \alpha^{3}}.
\end{align}

Finally, we have
\begin{align}
    \max _{||\vu||_{2}=1}\left|\left.\frac{d \tilde{V}_i(\alpha)}{d \alpha}\right|_{\alpha=0} \right| &\leq \frac{C'}{1-\gamma_i}+\frac{\gamma_i C'}{(1-\gamma_i)^2}=\frac{C'}{(1-\gamma_i)^2}
\end{align},
\begin{align}
    \max _{||\vu||_{2}=1}\left|\left.\frac{d^{2} \tilde{V}_i(\alpha)}{d \alpha^{2}}\right|_{\alpha=0} \right| &\leq \frac{C''}{1-\gamma_i}+\frac{2C'^2}{(1-\gamma_i)^2}+\left(\frac{2\gamma_i C'^2}{(1-\gamma_i)^3}+\frac{\gamma_i C''}{(1-\gamma_i)^2}\right)\notag\\
    &=\frac{C''}{(1-\gamma_i)^2}+\frac{2\gamma_i C'^2}{(1-\gamma_i)^3}
\end{align},
and
\begin{align}
    \max _{||\vu||_{2}=1}\left|\left.\frac{d^{3} \tilde{V}_i(\alpha)}{d \alpha^{3}}\right|_{\alpha=0} \right| &\leq \frac{C'''}{1-\gamma_i}+\frac{3\gamma_i C'C''}{(1-\gamma_i)^2}+3C'(\frac{2 \gamma_i^{2} C'^{2}}{(1-\gamma_i)^{3}}+\frac{\gamma_i C''}{(1-\gamma_i)^{2}})\notag\\
    &\hspace{100pt}+\frac{6 \gamma_i^{3} C'^{3}}{(1-\gamma_i)^{4}}+\frac{6\gamma_i^2 C'C''}{(1-\gamma_i)^{3}}+\frac{\gamma_i C'''}{(1-\gamma_i)^{2}}\notag\\
    &=\frac{C'''}{1-\gamma_i}+\frac{\gamma_i (6C'C''+C''')}{(1-\gamma_i)^{2}}+\frac{6\gamma_i^2 (C'^3+C'C'')}{(1-\gamma_i)^{3}}+\frac{6\gamma_i^3 C'^3}{(1-\gamma_i)^{4}}\notag\\
    &=\frac{C'''}{(1-\gamma_i)^2}+\frac{6\gamma_i C'C''}{(1-\gamma_i)^{3}}+\frac{6\gamma_i^2 C'^3}{(1-\gamma_i)^{4}}
\end{align}
\end{proof}

\begin{lemma}
Under the tabular softmax policy, $V^{\pi_{\theta}}_i(\mu)$ is Lipschitz, has a Lipschitz gradient and a Lipschtz Hessian for all $i$ and $\mu$, i.e.
\begin{align}
    &||V^{\pi_{\theta'}}_i(\mu)- V^{\pi_{\theta''}}_i(\mu)|| \leq \frac{2}{(1-\gamma_i)^2}||\theta'-\theta''||,\notag\\
    &||\nabla_{\theta'} V^{\pi_{\theta'}}_i(\mu)-\nabla_{\theta''} V^{\pi_{\theta''}}_i(\mu)|| \leq \frac{8}{(1-\gamma_i)^3}||\theta'-\theta''||,\,\text{and}\notag\\
    &||\nabla^2_{\theta'} V^{\pi_{\theta'}}_i(\mu)-\nabla^2_{\theta''} V^{\pi_{\theta''}}_i(\mu)|| \leq \frac{48}{(1-\gamma_i)^4}||\theta'-\theta''||.
\end{align}
\label{lem:V_softmax_Hessian_Lipschitz}
\end{lemma}

\begin{proof}
To show a function is Lipschitz, we show the derivative of the Hessian with respect to $\theta$ is bounded. Under the softmax parameterization, we have
\begin{align}
    \nabla_{\theta_{s}} \pi_{\theta}(a | s)=\pi_{\theta}(a | s)\left(e_{a}-\pi(\cdot | s)\right),
\end{align}
\begin{align}
    \nabla_{\theta_{s}}^{2} \pi_{\theta}(a | s)=\pi_{\theta}(a | s)\left(e_{a} e_{a}^{\top}-e_{a} \pi(\cdot | s)^{\top}-\pi(\cdot | s) e_{a}^{\top}+2 \pi(\cdot | s) \pi(\cdot | s)^{\top}-\operatorname{diag}(\pi(\cdot | s))\right),
\end{align}
\begin{align}
    \frac{\partial}{\partial \theta_{s,a'}}\nabla_{\theta_{s}}^{2} \pi_{\theta}(a | s)&=\pi_{\theta}(a | s)(\vct{1}(a=a')-\pi_{\theta}(a' | s))\left(e_{a} e_{a}^{\top}-e_{a} \pi(\cdot | s)^{\top}-\pi(\cdot | s) e_{a}^{\top}\right.\notag\\
    &\hspace{50pt}\left.+2\pi(\cdot | s) \pi(\cdot | s)^{\top}-\operatorname{diag}(\pi(\cdot | s))\right)\notag\\
    &\hspace{50pt}+\pi_{\theta}(a | s)(-e_a \pi_{\theta}(a' | s)e_{a'}^T+e_a\pi_{\theta}(a' | s)\pi_{\theta}(\cdot | s)^T-e_{a'} \pi_{\theta}(a' | s)e_{a}^T\notag\\
    &\hspace{50pt}+\pi_{\theta}(\cdot | s))\pi_{\theta}(a' | s)e_a^T+4\pi_{\theta}(\cdot | s)\pi_{\theta}(a' | s)e_{a'}^T-4\pi_{\theta}(\cdot | s)\pi_{\theta}\pi_{\theta}(\cdot | s)^T\notag\\
    &\hspace{50pt}+\text{diag}(\pi_{\theta}(a' | s)e_a)-\text{diag}(\pi_{\theta}(a' | s)\pi_{\theta}(\cdot | s)^T))
\end{align}

where $e_a$ is a vector with all 0 and 1 at action $a$. Then, for any $s$,
\begin{align}
    \sum_{a \in \mathcal{A}}\left|\left.\frac{d \pi_{\alpha}(a | s)}{d \alpha}\right|_{\alpha=0} \right| & \leq \sum_{a \in \mathcal{A}}\left|\left.\vu^{T} \nabla_{\theta+\alpha \vu} \pi_{\alpha}(a | s)\right|_{\alpha=0}\right| \notag\\ 
    & \leq \sum_{a \in \mathcal{A}} \pi_{\theta}(a | s)\left|\vu_{s}^{T} e_{a}-\vu_{s}^{T} \pi(\cdot | s)\right| \notag\\ 
    & \leq \max _{a \in \mathcal{A}}\left(\left|\vu_{s}^{T} e_{a}\right|+\left|\vu_{s}^{T} \pi(\cdot | s)\right|\right) \leq 2,
\end{align}
\begin{align} 
\sum_{a \in \mathcal{A}}\left|\left.\frac{d^{2} \pi_{\alpha}(a | s)}{d \alpha^{2}}\right|_{\alpha=0} \right| &\leq \sum_{a \in \mathcal{A}}\left|\left.\vu^{T} \nabla_{\theta+\alpha \vu}^{2} \pi_{\alpha}(a | s)\right|_{\alpha=0} \vu \right| \notag\\ 
&\leq \max _{a \in \mathcal{A}}\left(\left|\vu_{s}^{T} e_{a} e_{a}^{T} \vu_{s}\right|+\left|\vu_{s}^{T} e_{a} \pi(\cdot | s)^{T} \vu_{s}\right|+\left|\vu_{s}^{T} \pi(\cdot | s) e_{a}^{T} \vu_{s}\right|\right.
\notag\\ \hspace{50pt}&\left.+2\left|u_{s}^{\top} \pi(\cdot | s) \pi(\cdot | s)^{\top} u_{s}\right|+\left|u_{s}^{\top} \text{diag}(\pi(\cdot | s)) u_{s}\right|\right) \notag\\ 
&\leq6.
\end{align}

Similarly,
\begin{align} 
\sum_{a \in \mathcal{A}}\left|\left.\frac{d^{3} \pi_{\alpha}(a | s)}{d \alpha^{3}}\right|_{\alpha=0} \right| &\leq \sum_{a \in \mathcal{A}}\sum_{a'\in\Acal}\left|\left.\vu_{a'}\vu^{T} \nabla_{\theta+\alpha \vu}^{3} \pi_{\alpha}(a | s)\right|_{\alpha=0} \vu \right| \notag\\ 
&\leq 26 
\end{align}

Then we can use Lemma \ref{lem:bounded_thirdderivative} with $C'=2,C''=6,C'''=26$, and get 
\begin{align}
    &\max _{||\vu||_{2}=1}\left|\left.\frac{d \tilde{V}_i(\alpha)}{d \alpha}\right|_{\alpha=0} \right|\leq\frac{2}{(1-\gamma_i)^2},\notag\\
    &\max _{||\vu||_{2}=1}\left|\left.\frac{d^{2} \tilde{V}_i(\alpha)}{d \alpha^{2}}\right|_{\alpha=0} \right|\leq\frac{6}{(1-\gamma_i)^2}+\frac{8\gamma_i}{(1-\gamma_i)^3}\leq\frac{8}{(1-\gamma_i)^3},\notag\\
    &\max _{||\vu||_{2}=1}\left|\left.\frac{d^{3} \tilde{V}_i(\alpha)}{d \alpha^{3}}\right|_{\alpha=0}. \right|\leq\frac{26}{(1-\gamma_i)^2}+\frac{72\gamma_i}{(1-\gamma_i)^3}+\frac{48\gamma_i^2}{(1-\gamma_i)^4}\leq\frac{48}{(1-\gamma_i)^4}
    \label{eq:smoothness_V}
\end{align}

This is equivalent to
\begin{align}
    &||V_i^{\pi_{\theta'}}(\mu)- V_i^{\pi_{\theta''}}(\mu)|| \leq \frac{2}{(1-\gamma_i)^2}||\theta'-\theta''||,\notag\\
    &||\nabla V_i^{\pi_{\theta'}}(\mu)- \nabla V_i^{\pi_{\theta''}}(\mu)|| \leq \frac{8}{(1-\gamma_i)^3}||\theta'-\theta''||,\,\text{and}\notag\\
    &||\nabla^2 V_i^{\pi_{\theta'}}(\mu)-\nabla^2 V_i^{\pi_{\theta''}}(\mu)|| \leq \frac{48}{(1-\gamma_i)^4}||\theta'-\theta''||.
    \label{eq:Lipschitz_obj_grad_Hessian}
\end{align}
\end{proof}

\begin{lemma}
The cross entropy regularizer is Lipschitz, has a Lipschitz gradient and a Lipschtz Hessian, i.e.
\begin{align}
    &||\lambda\text{RE}(\pi_\theta')-\lambda\text{RE}(\pi_\theta'')|| \leq \lambda(\frac{1}{\sqrt{|\Acal|}}+1)||\theta'-\theta''||,\notag\\
    &||\nabla_{\theta'}\lambda\text{RE}(\pi_\theta')-\nabla_{\theta''}\lambda\text{RE}(\pi_\theta'')|| \leq \frac{2\lambda}{|\Scal|}||\theta'-\theta''||,\,\text{and}\notag\\
    &||\nabla^2_{\theta'}\lambda\text{RE}(\pi_\theta')-\nabla^2_{\theta''}\lambda\text{RE}(\pi_\theta'')|| \leq \frac{6\lambda}{|\Scal|}||\theta'-\theta''||.
\end{align}
\label{lem:RE_softmax_Hessian_Lipschitz}
\end{lemma}

\begin{proof}
Define
\begin{align}
    \zeta(\theta)\triangleq-\lambda\text{RE}(\pi_{\theta})=\frac{\lambda}{|\mathcal{S}||\mathcal{A}|} \sum_{s, a} \log \pi_{\theta}(a | s).
\end{align}

We have
\begin{align}
    &\nabla_{\theta_s}\zeta(\theta) = \frac{\lambda}{|\Scal|}(\frac{1}{|\Acal|}\vct{1}-\pi_{\theta}(\cdot|s)),\notag\\
    &\nabla_{\theta_s}^2\zeta(\theta) = \frac{\lambda}{|\Scal|}(-\text{diag}(\pi_{\theta}(\cdot|s))+\pi_{\theta}(\cdot|s)\pi_{\theta}(\cdot|s)^T),\notag\\
    &\frac{\partial}{\partial \theta_{s,a'}}\nabla_{\theta_s}^2\zeta(\theta) = \frac{\lambda}{|\Scal|}(-\pi_{\theta}(a'|s)e_{a'}e_{a'}^T+\pi_{\theta}(a'|s)\text{diag}(\pi_{\theta}(\cdot|s))\notag\\
    &\hspace{100pt}+2\pi_{\theta}(a'|s)\pi_{\theta}(\cdot|s)e_{a'}^T-2\pi_{\theta}(a'|s)\pi_{\theta}(\cdot|s)\pi_{\theta}(\cdot|s)^T).
\end{align}


Now we can bound the norm of the gradient, the norm of the Hessian, and the norm of the third level gradient.
\begin{align}
    ||\nabla_{\theta} \zeta(\theta)|| &=\sum_{s} ||\nabla_{\theta_{s}} \zeta(\theta)|| \notag\\
    &\leq \frac{\lambda}{|\Scal|}\sum_{s}||\frac{1}{|\Acal|}\vct{1}-\pi_{\theta}(\cdot|s)|| \notag\\
    &\leq \frac{\lambda}{|\Scal|}\sum_{s}\left(||\frac{1}{|\Acal|}\vct{1}||+||\pi_{\theta}(\cdot|s)||\right)\notag\\
    &\leq \frac{\lambda}{|\Scal|}\sum_{s}\left(\frac{1}{\sqrt{|\Acal|}}+1\right)\notag\\
    &\leq \lambda(\frac{1}{\sqrt{|\Acal|}}+1).
    \label{eq:RE_lipschitz}
\end{align}

For any vector $u\in\mathbb{R}^{|\Scal||\Acal}|$ with $||u||_2=1$,
\begin{align}
    \left|u^T \nabla_{\theta}^{2} \zeta(\theta) u\right|&=\left|\sum_{s} u_{s}^T \nabla_{\theta_{s}}^{2} \zeta(\theta) u_{s}\right| \notag\\
    &\leq \frac{\lambda}{|\Scal|} \sum_{s}\left|u_{s}^T\text{diag}(\pi_{\theta}(\cdot|s))u_s-u_{s}^T \pi_{\theta}(\cdot|s)\pi_{\theta}(\cdot|s)^T u_s\right| \notag\\
    &\leq \frac{2\lambda}{|\Scal|}\sum_{s}||u_s||_{\infty}^{2}\notag\\
    &\leq \frac{2\lambda}{|\Scal|}||u||_2^2\notag\\
    &\leq \frac{2\lambda}{|\Scal|},
    \label{eq:re_smooth}
\end{align}
where the first equality follows since $\nabla_{\theta_{s'}}\nabla_{\theta_{s''}}\zeta(\theta)=0,\,\forall s' \neq s''$. Using this method, we can further get
\begin{align}
    \left|\sum_{s',a'} u_{s',a'} u^T \nabla_{\theta}^{2} \zeta(\theta) u\right|&=\left|\sum_{s} \sum_{a'} u_{s,a'} u_{s}^T \nabla_{\theta_{s}}^{2} \zeta(\theta) u_{s}\right| \notag\\
    &\leq \frac{\lambda}{|\Scal|} \sum_{s}\left|-\sum_{a'}u_{s,a'}u_{s}^T\pi_{\theta}(a'|s)e_{a'}e_{a'}^T u_s\right.\notag\\
    &\hspace{50pt}+\sum_{a'}u_{s,a'}u_{s}^T \pi_{\theta}(a'|s)\text{diag}(\pi_{\theta}(\cdot|s)) u_s\notag\\
    &\hspace{50pt}+2\sum_{a'}u_{s,a'}u_{s}^T\pi_{\theta}(a'|s)\pi_{\theta}(\cdot|s)e_{a'}^T u_s \notag\\
    &\hspace{50pt}\left.-2\sum_{a'}u_{s,a'}u_{s}^T\pi_{\theta}(a'|s)\pi_{\theta}(\cdot|s)\pi_{\theta}(\cdot|s)^T u_s\right|\notag\\
    &\leq \frac{6\lambda}{|\Scal|}\sum_{s}||u_s||_{\infty}^{3}\notag\\
    &\leq \frac{6\lambda}{|\Scal|}||u||_3^3\notag\\
    &\leq \frac{6\lambda}{|\Scal|},
\end{align}
where the last inequality uses $||u||_3\leq||u||_2$. This implies that $\zeta(\theta)$ is $\lambda(\frac{1}{\sqrt{|\Acal|}}+1)$-Lipschitz, $\frac{2\lambda}{|\Scal|}$-smooth, and has $\frac{6\lambda}{|\Scal|}$-Lipschitz Hessian. 
\end{proof}

%% file: Appendix_Experiments.tex
\section{EXPERIMENTS DETAILS}

\subsection{Drone Experiments}
\label{sec:drone_experiment_details}
The framework used for the drone experiment is PEDRA \citep{PEDRA}, a 3D realistically stimulated drone navigation platform powered by Unreal Engine. In the simulated environment, a drone agent is equipped with a front-facing camera, and can implement actions to control its flight. To model the problem as an MDP, the state is represented by the monocular RGB images captured by the camera of the drone, which has dimension $103(height)\times 103(width)\times 3(color)$. There are a total number of 25 actions, corresponding to the drone controlling the yaw and pitch by various angles. Reward is calculated based on dynamic windowing of the simulated depth map, and is designed to encourage the drone to stay away from obstacles, as used in \citet{anwar2018navren}.

We select 4 indoor environments on the PEDRA platform: indoor long, indoor cloud, indoor frogeyes, and indoor pyramid. They contain widely different lighting conditions, wall colors, furniture objects, and hallway structures (Fig. \ref{fig:env_diagram}).

\begin{figure}[!ht]
\centering
  \includegraphics[width=\columnwidth]{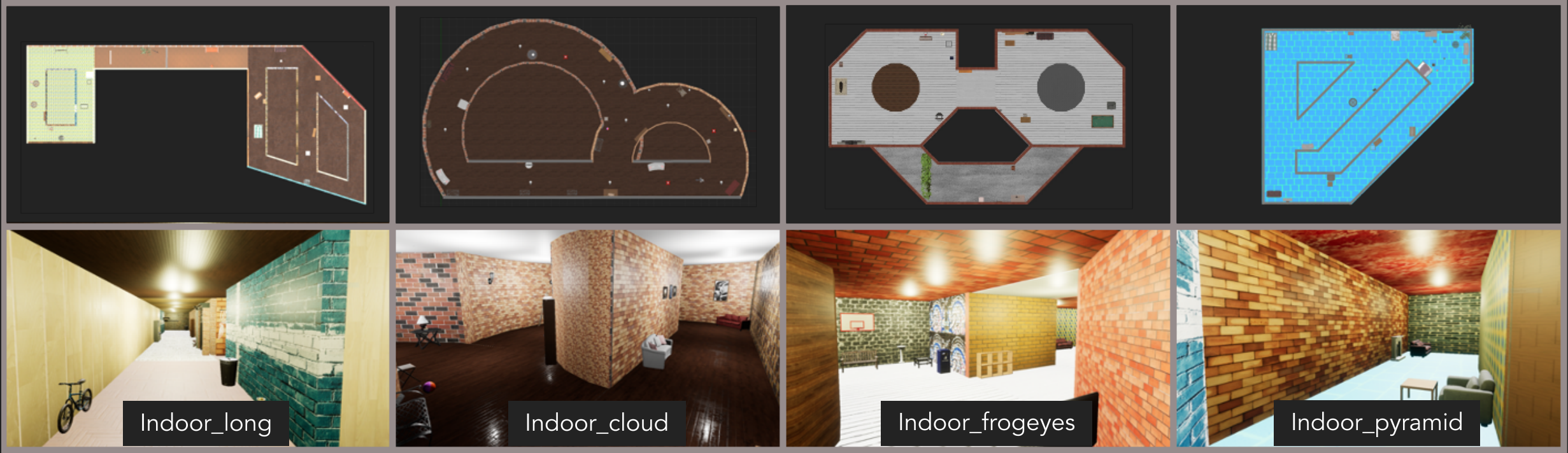}
  \caption{Environments used in drone navigation.}
  \label{fig:env_diagram}
\end{figure}

Every agent uses a 5-layer neural network as the function approximation. The exact architecture is shown in Figure \ref{fig:multienv_drone_architecture}. The agents use the ADAM optimizer with a constant step size of 1e-4, $\beta_1=0.9$, and $\beta_2=0.999$. Communication happens every episode, and follows a cyclic communication graph (ring graph). The same discount factor $\gamma=0.99$ is used by all agents. The weight of the cross entropy regularizer is chosen to be 0.03. We conducted three sets of experiments, where the local gradient $g_i^k$ is estimated using REINFORCE, advantage actor-critic (A2C), and proximal policy optimization (PPO), respectively. The discounted cumulative reward is estimated by the every visit Monte-Carlo method in all experiments. For PPO, we choose the clipping parameter $\epsilon$ to be 0.2. We train the agents for 4000 episodes in all experiments. Using two RTX2080 GPUs, each set of experiments takes about 25 hours to complete.

\begin{figure}[h]
\centering
  \includegraphics[width=0.9\columnwidth]{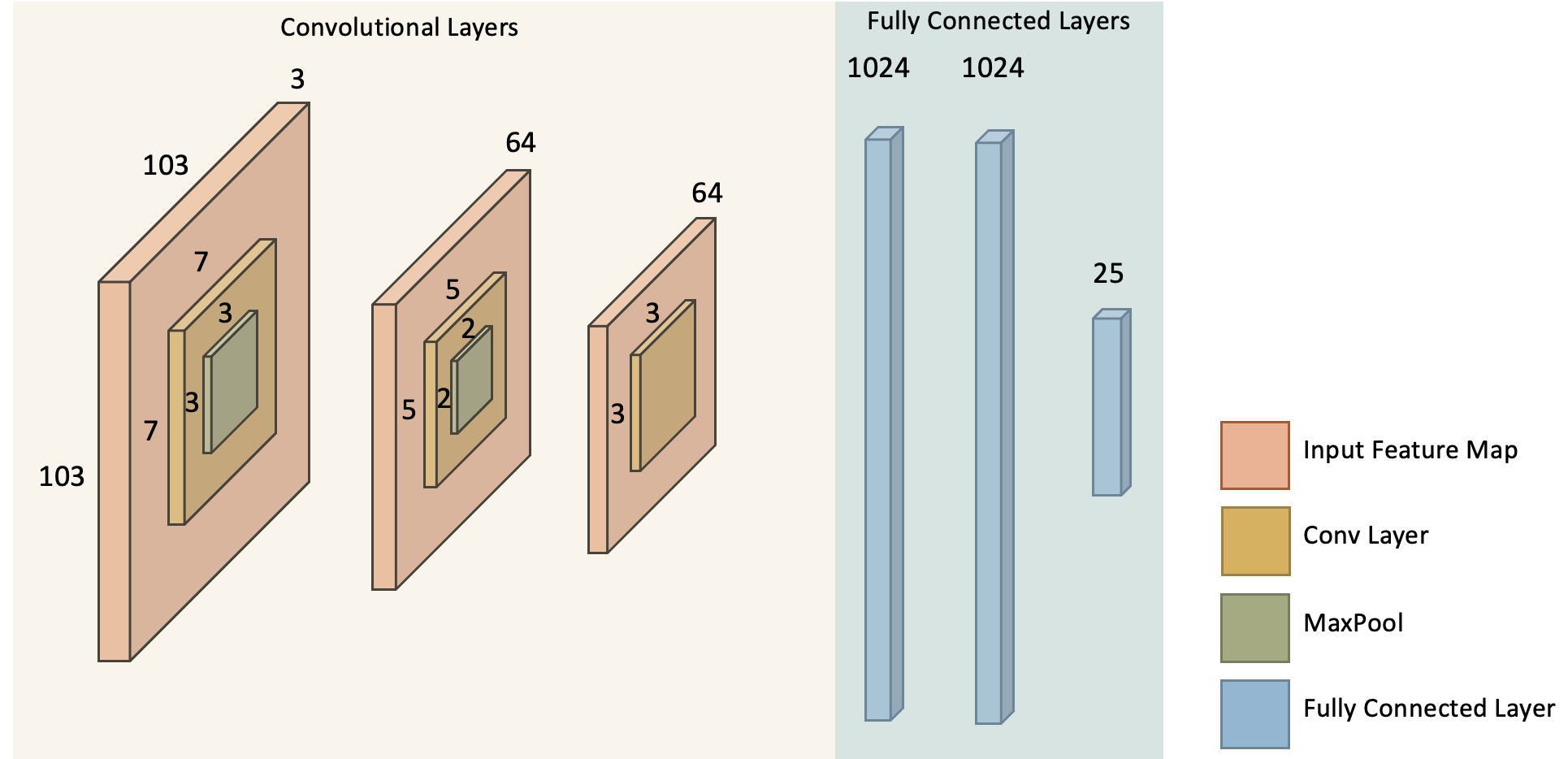}
  \caption{Network architecture for drone experiments}
  \label{fig:multienv_drone_architecture}
\end{figure}